\documentclass{new_tlp}
\usepackage{mathptmx}

\usepackage{latexsym}
\usepackage{amsmath}
\usepackage{amssymb}
\usepackage{alltt}
\usepackage{verbatim}
\usepackage{url}

\usepackage[cmyk]{xcolor}
\definecolor{color1}{cmyk}{0,100,0,0}
\definecolor{color2}{cmyk}{0,60,100,0}
\definecolor{color3}{cmyk}{100,0,0,0}
\definecolor{color4}{cmyk}{0,0,0,100}

\usepackage{pgfplots}
\pgfplotsset{compat=1.14}
\usepackage{tikz}
\usepackage[caption=false]{subfig}

\let\proof\relax

\usepackage{amsthm}

\newtheorem{lemma}{Lemma}
\theoremstyle{definition}
\newtheorem{example}{Example}
\newtheorem{definition}{Definition}[section]

\usepackage{booktabs}
\usepackage{todonotes}
\usepackage{thm-restate}

  \title[A Trajectory Calculus for Qualitative Spatial Reasoning Using ASP]
        {A Trajectory Calculus for Qualitative Spatial Reasoning Using Answer Set Programming}

  \author[G. Baryannis et al.]
         {GEORGE BARYANNIS, ILIAS TACHMAZIDIS, SOTIRIS BATSAKIS, GRIGORIS ANTONIOU\\
         University of Huddersfield, UK\\
         \email{\{g.bargiannis, i.tachmazidis, s.batsakis, g.antoniou\}@hud.ac.uk}
         \and MARIO ALVIANO\\
            University of Calabria, Italy\\
            \email{alviano@mat.unical.it}
         \and TIMOS SELLIS, PEI-WEI TSAI\\
            Swinburne University of Technology, Australia\\
            \email{\{tsellis, ptsai\}@swin.edu.au}}

\jdate{March 2018}
\pubyear{2018}
\pagerange{\pageref{firstpage}--\pageref{lastpage}}
\doi{S1471068401001193}

\begin{document}
%\nocite{*}% includes all entries of BibTeX database into the list of references.

\label{firstpage}

\maketitle

  \begin{abstract}
    Spatial information is often expressed using qualitative terms such as natural language expressions instead of coordinates; reasoning over such terms has several practical applications, such as bus routes planning. Representing and reasoning on trajectories is a specific case of qualitative spatial reasoning that focuses on moving objects and their paths. In this work, we propose two versions of a trajectory calculus based on the allowed properties over trajectories, where trajectories are defined as a sequence of non-overlapping regions of a partitioned map. More specifically, if a given trajectory is allowed to start and finish at the same region, 6 base relations are defined (TC-6). If a given trajectory should have different start and finish regions but cycles are allowed within, 10 base relations are defined (TC-10). Both versions of the calculus are implemented as ASP programs; we propose several different encodings, including a generalised program capable of encoding any qualitative calculus in ASP. All proposed encodings are experimentally evaluated using a real-world dataset. Experiment results show that the best performing implementation can scale up to an input of 250 trajectories for TC-6 and 150 trajectories for TC-10 for the problem of discovering a consistent configuration, a significant improvement compared to previous ASP implementations for similar qualitative spatial and temporal calculi. \textbf{This manuscript is under consideration for acceptance in TPLP.}
  \end{abstract}


%\tableofcontents

\section{Introduction}

In recent years, there has been a proliferation of widely available location-aware devices, leading to an abundance of location-based information. Such information has a spatio-temporal aspect that can be attached both directly and indirectly: in the former case it is provided by embedded systems like GPS, RFID or GSM, while in the latter it is provided by the users themselves when they, for instance, tag photos with geolocation information.

When this spatio-temporal information refers to the same object (at different, possibly sequential points in time), it essentially represents the \textit{trajectory} of a moving object. The ampleness and direct availability of such trajectory data is evident if we consider the billions of RFID tags generated daily or the millions of sensors collecting information within disparate devices and the fact that these figures are expected to double every 2 years over the next decade~\cite{Patroumpas2015}.

Trajectory data can prove useful in delivering a wide array of location-based services. The most common category of applications is related to traffic surveillance and control, and fleet management, with trajectories being used to detect flocks and convoys or to identify frequently followed routes. More personalised trajectory-based services can include carpooling via similarity identification in vehicle traces or tracing services for objects or people. Other possible applications include wildlife protection by monitoring the clustering of animals.

A common characteristic of the aforementioned trajectory-based services is that they rely on a means of comparing between different trajectories. Traditionally, this can be achieved in one of two ways: quantitatively or qualitatively. The former approach involves directly comparing points within trajectories in terms of their spatial or temporal aspects while the latter depends on expressing relations between moving objects and their trajectories using natural language expressions. Qualitative reasoning has been argued to be closer to the way the human mind reasons and makes decisions, in both spatial and temporal settings~\cite{Renz2000,Escrig2002}. Since the majority of the trajectory-based services outlined above are directly aimed at information systems interacting with users, qualitative reasoning can be deemed more suitable.

In this work, we introduce two qualitative calculi, named TC-6 and TC-10, for trajectories that are defined based on arbitrarily partitioned maps. Both calculi are designed with applicability in mind, facilitating straightforward implementations that can scale to large numbers of trajectories. To that end, we rely on three important simplifications:
\begin{itemize}
\item Each trajectory is modelled as a sequence of regions on the map, instead of a sequence of latitude/longitude pairs.
\item Trajectories are compared as a whole and not on the basis of individual positions of moving objects.
\item Characteristics of moving objects such as their velocity and acceleration are not taken into account.
\end{itemize}

%These simplifications make
TC-6 and TC-10 are not focused on real-time and predictive applications that require reasoning on specific characteristics of moving objects at any given point in time (e.g velocity and acceleration), such as collision monitoring and prevention. The proposed calculi are especially suitable for any application that relies on trajectory databases to reason about the relations among large numbers of trajectories such as Origin/Destination (O/D) analysis~\cite{Hu2016,Andrienko2017}, transportation demand discovery~\cite{Wang2017,Moreira2013} and stream of population monitoring~\cite{Ma2017}.

TC-6 represents the simplest case where no restrictions are imposed on trajectories apart from the fact that they need to represent moving objects, i.e., ones that do not stay on the same region in consecutive time points. TC-10, on the other hand, requires that moving objects do not finish at the same region as the one they started at. For both calculi, we capture all possible pairwise disjoint relations between two trajectories and define all possible results of composing these base relations, in the form of composition tables. Various alternative Answer Set Programming (ASP) implementations are proposed for each calculus and a comparative evaluation is presented, focusing on performance in terms of execution time and memory consumption for consistency problems based on a real dataset. The results indicate significant improvements in terms of execution time and memory consumption, compared to standard ASP implementations of qualitative calculi in literature.
As an additional result, we propose a generalised ASP encoding for qualitative calculi using extensional predicates to represent relations and composition tables.
The meta-encoding can enforce the composition table by means of a single integrity constraint, while all other encodings use at least one rule for each entry in the table.
On the other hand, such a generalisation comes with some overhead, which is also measured in our experiments.

The rest of this paper is structured as follows. Section~\ref{sec:relwork} offers a concise analysis of related efforts on qualitative reasoning about moving objects, while Section~\ref{sec:qc} defines the computational problem analysed in this work. Sections~\ref{sec:tc6} and~\ref{sec:tc10} formalise TC-6 and TC-10, respectively, in terms of their base relations and the associated composition table. Section~\ref{sec:impl} proposes a series of ASP encodings for the two calculi, as well as a generalised ASP encoding for qualitative calculi, while Section~\ref{sec:eval} discusses the experimental evaluation of the proposed encodings. Finally, Section~\ref{sec:concl} concludes and points out future research directions.

\section{Related Work}
% * <ptsai@swin.edu.au> 2017-08-07T15:30:25.772Z:
%
% In my opinion, this section is well organized, already.
%
% ^.
\label{sec:relwork}

A concise summary of research on qualitative spatio-temporal relations for moving objects is provided by Van de Weghe~\citeyear{Weghe2017}. The author makes a distinction between the so-called ``high-level'' and ``low-level'' approaches: the former focus only on modelling spatio-temporal continuity, while the latter concern themselves with a detailed representation of moving objects that takes into account properties such as objects' orientation and velocity.

Earliest works are representative of the ``high-level'' category, such as the work of Muller~\citeyear{Muller1998}, where six different classes of motion are defined between two objects. Then, considering that the two objects are related using one of the spatial and temporal relations in RCC-8~\cite{Cohn1997} and Allen's interval algebra~\cite{Allen1981}, respectively, the author determines the new spatial and temporal relations that would result if one of the six types of motion takes place. This approach is further refined by Hazarika and Cohn~\citeyear{Hazarika2001} to disallow abnormal transitions, e.g., objects that disappear and reappear instantaneously at the same spatial location.

Later approaches to qualitative spatio-temporal reasoning focus on integrating more characteristics of moving objects. The Qualitative Trajectory Calculus (QTC) is a prominent example and has yielded several variants in literature so far. QTC\textsubscript{B}~\cite{Weghe2006} is the basic type and relies on the Euclidean distance between moving objects, as well as their speed. While QTC\textsubscript{B} assumes objects moving freely in a n-dimensional space, QTC\textsubscript{N}~\cite{Delafontaine2011} focuses on objects whose trajectories are constrained by a network (e.g., a road network). Another variant, QTC\textsubscript{C}~\cite{Weghe2005}, is a more complex version of the basic type that also takes into account the motion azimuth of the objects and relies on the double cross notation introduced by Zimmermann and Freksa~\citeyear{Zimmermann1996}.

The detailed representation of moving objects inherent in all variants of QTC leads to significant numbers of relations: QTC\textsubscript{B} and QTC\textsubscript{N} define 27 relations, while QTC\textsubscript{C} includes 81 relations in total. This, in turn, leads to limitations with regard to automated reasoning, since at its most complex form, QTC reasoning relies on a 81x81 composition table. This is evident in the only available (to the best of our knowledge) implementation of QTC~\cite{Delafontaine2011B}, where reasoning capabilities amount to calculating the relation between given trajectory pairs at any given time, while composition of relations is not supported.

%The case of~\cite{Noyon2007} is simpler, provided all objects are points, which yields 9 relations. (removed since it is mentioned below)

Following a similar approach, Noyon et al.~\citeyear{Noyon2007} propose a trajectory data model that relies on the relative velocity and relative position between two moving objects. However, moving objects are not restricted to points, but can also be lines or polygons. The resulting relations depend on whether the objects are moving closer or farther, whether they are accelerating or decelerating and whether they are points, lines or polygons. This simplified model results in significantly less relations (compared to all the aforementioned QTC variants), ranging from 6 relations for points to 9 relations for polygons. However, no implementation has been proposed so far to validate the model's applicability.

It should be evident that literature on qualitative spatio-temporal relations for moving objects has mostly focused on creating models that allow for a detailed representation of the objects in motion, rather than provide support for efficient reasoning mechanisms. In order to achieve the latter, models need to be less elaborate either through simplification or by focusing on one feature at a time. In their attempts to provide a general framework for qualitative spatio-temporal representation and reasoning, Mart\'{i}nez-Mart\'{i}n et al.~\citeyear{Martinez2012} follow the second path, providing CSP formalisations~\cite{Freuder2006} that only focus on one feature of moving objects at a time. On the contrary, our work follows the first path: by applying the simplifications outlined in the introductory section, a high-level model is proposed which does not represent the individual features of objects in motion but, instead, is able to efficiently reason about the relations among trajectories when viewed as complete paths, scaling up to hundreds of trajectories. In that sense our work is not directly comparable to~\cite{Martinez2012}, since their reasoning process focuses exclusively on one of distance, velocity or acceleration features of a moving object, while ours proposes a trajectory model that abstracts away these features.

To the best of our knowledge the only studies in literature that employ ASP for qualitative spatio-temporal reasoning are those of Li~\citeyear{Li2012}, Brenton et al.~\citeyear{Brenton2016} and Walega et al.~\citeyear{Walega2017}; out of these, only the first two are directly related to our work, since Walega et al.~\citeyear{Walega2017} focuses on reasoning about action and change in relation to space. We use the best-performing encoding in Brenton et al.~\citeyear{Brenton2016}, which represents an improvement over Li~\citeyear{Li2012}, as a starting point for our implementation and create several new encodings that outperform it. Additionally, we propose a generalised and compact meta-encoding for spatio-temporal calculi that takes advantage of conditional literals~\cite{Gebser2015} while also being able to enforce a composition table relying only on a single integrity constraint and a set of facts representing table entries.

\section{Qualitative Calculi and Model Existence}
\label{sec:qc}

This work focuses on qualitative calculi for reasoning on relations among elements of a domain of interest.
A calculus is specified via a \emph{composition table}, that is, a function mapping every pair of relations with the relations that can be associated with their composition.
The computational problem analysed in this paper is \emph{model existence}, defined next.

\begin{definition}
\label{def:problem}
({\bf QC Model existence})
Given a set $\mathcal{E}$ of elements, a set $\mathcal{R}$ of relations including $\mathit{eq}$, a composition table $c : \mathcal{R} \times \mathcal{R} \to 2^\mathcal{R} \setminus \{\emptyset\}$, and a set $\mathcal{C}$ of constraints of the form
$(x,y) \in R$,
where $x,y \in \mathcal{E}$, and $R \subseteq \mathcal{R}$,
\emph{QC model existence} amounts to checking the existence of a model of $(\mathcal{E},\mathcal{R},c,\mathcal{C})$, that is, an assignment
$\nu : \mathcal{E} \times \mathcal{E} \to \mathcal{R}$
satisfying the following properties:
(i) for each $x \in \mathcal{E}$, $\nu(x,x) = \mathit{eq}$;
(ii) for each $x,y,z \in \mathcal{E}$, $\nu(x,z) \in c(\nu(x,y), \nu(y,z))$ holds;
(iii) for each constraint $(x,y) \in R$ in $\mathcal{C}$, $\nu(x,y) \in R$ holds.
\end{definition}

\begin{restatable}{theorem}{ThmNP}
QC model existence belongs to the complexity class NP.
\end{restatable}

The problem can be solved by an ASP encoding comprising choice rules and integrity constraints, a fragment of ASP for which answer set existence belongs to the complexity class NP  \cite{DBLP:conf/slp/MarekT89,DBLP:journals/jlp/CadoliS93}.
Intuitively, such an encoding guesses a unique relation for each pair of elements by means of a choice rule, and enforces the satisfaction of the composition table and of the constraints in input by means of integrity constraints.

%Encodings and proofs are given in appendix.

\section{Trajectory Calculus with 6 Base Relations (TC-6)}
\label{sec:tc6}

Both calculi defined in this work assume that trajectories are expressed based on a partitioning of a given map according to the following definition:

\begin{definition}
\label{def:partitioning}
({\bf Partitioning}) Given a map $M$, a partitioning $R$ of the map $M$ is defined as a set
of non-overlapping regions $r_{i}$, as defined in RCC8~\cite{Renz2002}, such that $M = \bigcup\limits_{r_{i} \in R}^{} r_{i}$. Following RCC8 notation, each region is related through EQ (equal) only to itself, EC (externally connected) to its neighbouring regions, and DC (disconnected) with all other regions.
\end{definition}

For the first calculus, namely TC-6, trajectories are arbitrary sequences of regions as expressed in Definition~\ref{def:partitioning}, with the sole constraint that consecutive regions within them must be different and externally connected:

\begin{definition}
\label{def:tc6traj}
({\bf TC-6 Trajectory}) Given a partitioning $R$, a trajectory $T$ in TC-6 is defined as a sequence of
externally connected regions ($t_{1}, t_{2}, ..., t_{n}$), $n \geq 2$ where $t_{i} \neq t_{i+1}$ and
EC($t_{i},t_{i+1}$), $1 \leq i < n$.
\end{definition}

Based on this definition, TC-6 admits trajectories that, for instance, start and finish at the same region or revisit a region after moving to a different one.

\begin{definition}{}
({\bf TC-6 Base relations})
The 6 base relations of Table~\ref{tab:trajectory_relations_6}, capture the possible
relations between two trajectories $T_{1}$ and $T_{2}$, where $|T_{i}|$ denotes the length
of trajectory $T_{i}$, while $T_{i}^{j}$ denotes the region $t_{j}$ of trajectory $T_{i}$.
All 6 base relations are pairwise disjoint and symmetric. Relation Equal (Eq) is also
transitive.
\end{definition}

\begin{definition}{}
\label{def:tc6comp}
({\bf TC-6 Composition table})
Table~\ref{tab:composition_table_6} represents the result of the composition of each pair of
trajectory relations of Table~\ref{tab:trajectory_relations_6}. The composition table is
interpreted as follows: if relation i holds between $T_{1}$ and $T_{2}$ while relation j
holds between $T_{2}$ and $T_{3}$, then the entry (i,j) of Table~\ref{tab:composition_table_6}
denotes the possible relation(s) holding between $T_{1}$ and $T_{3}$. Note that relation All represents all 6 base relations.
\end{definition}

\begin{table}[ht]
\caption{TC-6 Base relations}
\label{tab:trajectory_relations_6}
\centering
%\small
\begin{tabular}{lll}
	\hline\hline
	Relation & Definition & Interpretation \\ [0.5ex]
	\hline\hline
	Equal (Eq) & $|T_{1}| = |T_{2}| = n,$ & $T_{1}$ and $T_{2}$ are equal \\
	 & $EQ(T_{1}^{i}, T_{2}^{i}), 1 \leq i \leq n$ & (identical trajectories) \\ \hline
	
	Alternative (Alt) & $|T_{1}| = n, |T_{2}| = m,$ & $T_{1}$ and $T_{2}$ are alternative\\
	 & $EQ(T_{1}^{1},T_{2}^{1}), EQ(T_{1}^{n},T_{2}^{m}),$ &  (different trajectories for \\
	 & EITHER $n \neq m$ & the same start and finish\\
	 & OR $\exists i~NOT~EQ(T_{1}^{i},T_{2}^{i}), 1 < i < n$ &  regions)\\ \hline
 %	 & ~~~~~~~$$ & \\ \hline
	
	Start (S) & $|T_{1}| = n, |T_{2}| = m,$ & $T_{1}$ and $T_{2}$ start at the same region\\
	 & $EQ(T_{1}^{1},T_{2}^{1}), NOT~EQ(T_{1}^{n},T_{2}^{m})$ &   (but finish at different regions)\\ \hline
%	 & $ $ & \\ \hline
	
	Finish (F) & $|T_{1}| = n, |T_{2}| = m,$ & $T_{1}$ and $T_{2}$ finish at the same region \\
	 & $NOT~EQ(T_{1}^{1},T_{2}^{1}), EQ(T_{1}^{n},T_{2}^{m})$ & (but start at different regions)\\ \hline
	% & $$ & \\ \hline

	Intersect (I) & $|T_{1}| = n, |T_{2}| = m,$ & $T_{1}$ and $T_{2}$ intersect\\
	 & $NOT~EQ(T_{1}^{1},T_{2}^{1}), NOT~EQ(T_{1}^{n},T_{2}^{m}),$ & (at least one common\\
	 & $\exists i,j~EQ(T_{1}^{i},T_{2}^{j}),$ & region, but different\\
	 & EITHER $1 {\leq} i {<} n,1 {<} j {\leq} m$ & start and finish regions)\\
	 & OR $1 {<} i {\leq} n,1 {\leq} j {<} m$ & \\	\hline
 %	 & ~~~~~ & \\\hline

	Disjoint (Dis)& $|T_{1}| = n, |T_{2}| = m,$ & $T_{1}$ and $T_{2}$ are disjoint \\
	 & $\forall i,j~NOT~EQ(T_{1}^{i},T_{2}^{j}), 1 {\leq} i {\leq} n,1 {\leq} j {\leq} m$ & (no common regions) \\\hline
%	 & $$ & \\ \hline 	
	 %& $1 \leq j \leq m$ &  \\ \hline	
	 \hline	
\end{tabular}
\end{table}

\begin{table}[ht]
\caption{TC-6 Composition Table}
\centering
%\small
%\aboverulesep=0ex
%\belowrulesep=0ex
%\begin{tabular}{ | l || l | l | l | l | l | l |}
\begin{tabular}{  l  l  l  l  l  l  l }
	\hline\hline%\cmidrule{1-7}
	Relations & Eq & Alt & S & F & I & Dis \\  [0.5ex]
	\hline%\cmidrule{1-7}
    %\morecmidrules
    %\cmidrule{1-7}
	Eq & Eq & Alt & S & F & I & Dis \\ %\cmidrule{1-7}
	Alt & Alt & Eq, Alt & S & F & I, Dis & I, Dis \\ %\cmidrule{1-7}
%	S & S & S & Eq,   & I,Dis & F,I, & F,I, \\		
%      &   &   & Alt,S &       & Dis  & Dis \\ \hline
	S & S & S & Eq, Alt, S   & I,Dis & F, I, Dis & F, I, Dis\\ %\cmidrule{1-7}
%	F & F & F & I,Dis & Eq,   & S,I, & S,I, \\		
%      &   &   &       & Alt,F & Dis   & Dis \\ \hline		
	F & F & F & I, Dis & Eq, Alt, F   & S, I, Dis & S, I, Dis \\ %\cmidrule{1-7}
%	I & I & I,Dis & F,I, & S,I & All & Alt,S,  \\
%      &   &       & Dis  & Dis &       & F,I,Dis  \\ \hline
	I & I & I, Dis & F, I, Dis & S, I, Dis & All & Alt, S, F, I, Dis  \\%\cmidrule{1-7}
%	Dis & Dis & I,Dis & F,I, & S,I, & Alt,S, & All \\
%        &       &     & Dis  & Dis  & F,I,Dis & \\ \hline	
	Dis & Dis & I, Dis & F, I, Dis & S, I, Dis & Alt, S, F, I, Dis & All \\ %\cmidrule{1-7}
    \hline\hline
\end{tabular}
\label{tab:composition_table_6}
\end{table}

\begin{comment}
\begin{definition}
({\bf TC-6 Model existence})
Given a set $\mathcal{T}$ of trajectories, and a set $\mathcal{C}$ of constraints of the form
$(T,T') \in \mathcal{R}$,
where $T,T' \in \mathcal{T}$, and $\mathcal{R} \subseteq \{Eq,$ $Alt,$ $S,$ $F,$ $I,$ $Dis\}$,
\emph{TC-6 model existence} amounts to checking the existence of a model of $(\mathcal{T},\mathcal{C})$, that is, an assignment
$\nu : \mathcal{T} \times \mathcal{T} \to \{Eq, Alt, S, F, I, Dis\}$
such that, for each constraint $(T,T') \in \mathcal{R}$ in $\mathcal{C}$, $\nu(T,T') \in \mathcal{R}$ holds.
\end{definition}
\end{comment}

\begin{example}
Let $\mathcal{T}$ be $\{T_1,T_2,T_3\}$, and let $\mathcal{C}$ be $\{(T_1,T_2) \in \{Dis\}, (T_2,T_3) \in \{Eq, Alt\}\}$.
Every model of $(\mathcal{T},\mathcal{C})$ maps $(T_1,T_2)$ to $Dis$.
Moreover, the model mapping $(T_2,T_3)$ to $Eq$ also maps $(T_1,T_3)$ to $Dis$.
Finally, there are two models mapping $(T_2,T_3)$ to $Alt$, one mapping $(T_1,T_3)$ to $I$, and one mapping $(T_1,T_3)$ to $Dis$.
\hfill $\blacksquare$
\end{example}

\section{Trajectory Calculus with 10 Base Relations (TC-10)}
\label{sec:tc10}

Retaining the same way of map partitioning as in Definition~\ref{def:partitioning}, we impose a further restriction on trajectories, requiring them to have different start and finish regions. This leads us to a new calculus, namely TC-10, formalised by the following definitions.

%\begin{definition}{}
%({\bf Partitioning}) Given a map $M$, a partitioning $R$ of the map $M$ is defined as a set
%of non-overlapping regions $r_{i}$, such that $M = \bigcup\limits_{r_{i}
%\in R}^{} r_{i}$.
%\end{definition}

\begin{definition}{}
({\bf TC-10 Trajectory}) Given a partitioning $R$, a trajectory $T$ in TC-10 is defined as a sequence of externally connected
regions ($t_{1}, t_{2}, ..., t_{n}$), $n \geq 2$ where $t_{1} \neq t_{n}$, $t_{i} \neq t_{i+1}$
and EC($t_{i},t_{i+1}$), $1 \leq i < n$.
\end{definition}

Based on this definition, TC-10 admits trajectories that have any arbitrary cycle within them, but never finish at the same region they started.

\begin{definition}{}
({\bf TC-10 Base relations})
The 10 base relations of Table~\ref{tab:trajectory_relations_10} capture the possible
relations between two trajectories $T_{1}$ and $T_{2}$, where $|T_{i}|$ denotes the length
of trajectory $T_{i}$, while $T_{i}^{j}$ denotes the region $t_{j}$ of trajectory $T_{i}$.
All 10 base relations are pairwise disjoint. Relations Equal (Eq), Reverse (Rev), Alternative
(Alt), Return (Ret), Start (S), Finish (F), Intersect (I) and Disjoint (Dis) are symmetric.
Relation Equal (Eq) is also transitive. Relations Extends (Ex) and Extended By (Exi) are inverses of each other.
\end{definition}

\begin{definition}{}
\label{def:tc10comp}
({\bf TC-10 Composition table})
Table~\ref{tab:composition_table_10} represents the result of composing each distinct pair of
trajectory relations of Table~\ref{tab:trajectory_relations_10}. The composition table is
interpreted as follows: if relation i holds between $T_{1}$ and $T_{2}$ while relation j
holds between $T_{2}$ and $T_{3}$, then the entry (i,j) of Table~\ref{tab:composition_table_10}
denotes the possible relation(s) holding between $T_{1}$ and $T_{3}$. Note that relation All represents all 10 base relations.
\end{definition}

\begin{table}[ht]
\caption{TC-10 Base relations}
\centering
%\small
\begin{tabular}{ l l l } %| l | l | l |}
	\hline\hline
	Relation & Definition & Interpretation \\ [0.5ex]
	\hline\hline

%	Equal (Eq) & $|T_{1}| = |T_{2}| = n,$ & $T_{1}$ and $T_{2}$ are equal \\
%	 & $T_{1}^{i} = T_{2}^{i}, 1 \leq i \leq n$ & (identical trajectories) \\ \hline

	Equal (Eq) & As in Table~\ref{tab:trajectory_relations_6} & As in Table~\ref{tab:trajectory_relations_6} \\ \hline
	
    Reverse (Rev) & $|T_{1}| = |T_{2}| = n,$ & $T_{1}$ and $T_{2}$ are reverse \\
	 & $EQ(T_{1}^{i},T_{2}^{n+1-i}), 1 \leq i \leq n$ & (reversed trajectories) \\ \hline
	     % &  & \\ \hline
	
%	Alternative & $|T_{1}| = n, |T_{2}| = m,$ & $T_{1}$ and $T_{2}$ are\\
%	(Alt) & $T_{1}^{1} = T_{2}^{1}, T_{1}^{n} = T_{2}^{m},$ & alternative (different\\
%	 & EITHER $n \neq m$ & trajectories for the\\
%	 & OR $\exists i~T_{1}^{i} \neq T_{2}^{i},$ & same start and\\
%	 & ~~~~~~~$1 < i < n$ & finish regions)\\ \hline

 	Alternative (Alt) & As in Table~\ref{tab:trajectory_relations_6} & As in Table~\ref{tab:trajectory_relations_6} \\ \hline

    Return (Ret) & $|T_{1}| = n, |T_{2}| = m,$ & $T_{1}$ and $T_{2}$ are return \\
	 & $EQ(T_{1}^{1},T_{2}^{m}), EQ(T_{1}^{n},T_{2}^{1}), $ & trajectories (can be  \\
	 & EITHER $n \neq m$ & used in order to return \\
	 & OR $\exists i~NOT~EQ(T_{1}^{i},T_{2}^{n+1-i}), 1 < i < n$ &  to the same region) \\ \hline
	% &  &  \\
%	 & ~~~~~ &  \\	
%	 & ~~~~~$$ &  \\ \hline
	
%	Start (S) & $|T_{1}| = n, |T_{2}| = m,$ & $T_{1}$ and $T_{2}$ start at the\\
%	 & $T_{1}^{1} = T_{2}^{1}, T_{1}^{n} \neq T_{2}^{m}$ &  same region (but finish\\
%	 & & at different regions)\\ \hline
	
	Start (S) & As in Table~\ref{tab:trajectory_relations_6} & As in Table~\ref{tab:trajectory_relations_6} \\ \hline
	
%	Finish (F) & $|T_{1}| = n, |T_{2}| = m,$ & $T_{1}$ and $T_{2}$ finish at the \\
%	 & $T_{1}^{1} \neq T_{2}^{1}, T_{1}^{n} = T_{2}^{m}$ & same region (but start\\
%	 & & at different regions)\\ \hline

    Finish (F) & As in Table~\ref{tab:trajectory_relations_6} & As in Table~\ref{tab:trajectory_relations_6} \\ \hline

    Extends (Ex) & $|T_{1}| = n, |T_{2}| = m,$ & $T_{1}$ extends $T_{2}$ (the finish region  \\
	 & $NOT~EQ(T_{1}^{n},T_{2}^{1}), EQ(T_{1}^{1},T_{2}^{m})$ & of $T_{2}$ is the start region of $T_{1}$)  \\ \hline
%	 & $$ &  \\ \hline	

	Extended By (Exi) & $|T_{1}| = n, |T_{2}| = m,$ & $T_{1}$ is extended by $T_{2}$ (the finish \\
	 & $NOT~EQ(T_{1}^{1},T_{2}^{m}), EQ(T_{1}^{n},T_{2}^{1}) $ &  region of $T_{1}$ is the start region of $T_{2}$)\\ \hline
%	 & $$ & \\ \hline	

	Intersect & $|T_{1}| = n, |T_{2}| = m,$ & $T_{1}$ and $T_{2}$ intersect\\
	(I) & $NOT~EQ(T_{1}^{1},T_{2}^{1}), NOT~EQ(T_{1}^{n},T_{2}^{m}),$ & (at least one common\\
 	 & $NOT~EQ(T_{1}^{1},T_{2}^{m}), NOT~EQ(T_{1}^{n},T_{2}^{1}),$ & region, but different\\
     & $\exists i,j~EQ(T_{1}^{i},T_{2}^{j}),$ & start and finish regions) \\
     & EITHER $1 {<} i {<} n,1 {\leq} j {\leq} m$ & \\
	 & OR $1 {\leq} i {\leq} n,1 {<} j {<} m$ & \\ \hline
%	 &  & \\	
%	 & ~~~~~$$ & \\
 %	 &  & \\
 %	 & ~~~~~$$ & \\\hline

%	Disjoint & $|T_{1}| = n, |T_{2}| = m,$ & $T_{1}$ and $T_{2}$ are disjoint \\
%	(Dis) & $\forall i,j~T_{1}^{i} \neq T_{2}^{j},$ & (no common regions) \\
%	 & $1 {\leq} i {\leq} n,1 {\leq} j {\leq} m$ & \\ \hline 	
%	 %& $1 \leq j \leq m$ &  \\ \hline	
	
	Disjoint (Dis) & As in Table~\ref{tab:trajectory_relations_6} & As in Table~\ref{tab:trajectory_relations_6} \\ \hline\hline
%	 & & \\ \hline
\end{tabular}
\label{tab:trajectory_relations_10}
\end{table}

\begin{table}[ht]
\caption{Composition Table for 10 Trajectory Relations}
\centering
\small
\renewcommand{\arraystretch}{0.85}
%\aboverulesep=0ex
%\belowrulesep=0ex
%\begin{tabular}{ | l || l | l | l | l | l | l | l | l | l | l |}
\begin{tabular}{  l  l  l  l  l  l  l  l  l  l  l }
	\hline\hline%\cmidrule{1-11}
	Rels & Eq & Rev & Alt & Ret & S & F & Ex & Exi & I & Dis \vspace{-0.6em}\\  [0.5ex]
	\hline\hline%\cmidrule{1-11}
    %\morecmidrules
    %\cmidrule{1-11}
	Eq & Eq & Rev & Alt & Ret & S & F & Ex & Exi & I & Dis \vspace{-0.6em}\\\hline %\cmidrule{1-11}
	Rev & Rev & Eq & Ret & Alt & Exi & Ex & F & S & I & Dis \vspace{-0.6em}\\\hline %\cmidrule{1-11}
	Alt & Alt & Ret & Eq,  & Rev, & S & F & Ex & Exi & I,Dis & I,Dis \\
	    &     &     & Alt  & Ret &   &   &    &     &       &       \vspace{-0.6em}\\\hline %\cmidrule{1-11}
	Ret & Ret & Alt & Rev, & Eq,  & Exi & Ex & F & S & I,Dis & I,Dis \\
        &     &     & Ret  & Alt &     &    &   &   &       &   \vspace{-0.6em}\\\hline %\cmidrule{1-11}
	S & S & Ex & S & Ex & Eq,Alt, & Exi,I, & Rev,Ret, & F,I, & F,Exi, & F,Exi, \\	
	  &   &    &   &    & S       & Dis    & Ex       & Dis  & I,Dis  & I,Dis \vspace{-0.6em}\\\hline %\cmidrule{1-11}		
	F & F & Exi & F & Exi & Ex,I, & Eq,Alt, & S,I, & Rev,Ret, & S,Ex, & S,Ex, \\ 	
	  &     &   &   &     & Dis   & F       & Dis  & Exi      & I,Dis & I,Dis \vspace{-0.6em}\\\hline %\cmidrule{1-11}
	Ex & Ex & S & Ex & S & F,I, & Rev,Ret, & Exi,I, & Eq,Alt, & F,Exi, & F,Exi,  \\
	   &    &   &    &   & Dis  & Ex       & Dis    & S       & I,Dis  & I,Dis  \vspace{-0.6em}\\\hline %\cmidrule{1-11}
	Exi & Exi & F & Exi & F & Rev,Ret, & S,I, & Eq,Alt, & Ex,I, & S,Ex, & S,Ex, \\
	    &     &   &     &   & Exi      & Dis  & F       & Dis   & I,Dis & I,Dis \vspace{-0.6em}\\\hline %\cmidrule{1-11}
	I & I & I & I,   & I,   & F,Ex, & S,Exi, & S,Exi, & F,Ex, & All & Alt,Ret, \\
	  &   &   & Dis  & Dis & I,Dis & I,Dis  & I,Dis  & I,Dis &     & S,F,Ex, \\ 	
	  &   &   &     &     &       &        &        &       &     & Exi,I,Dis \vspace{-0.5em}\\\hline %\cmidrule{1-11}
	Dis & Dis & Dis & I,   & I,   & F,Ex, & S,Exi, & S,Exi, & F,Ex, & Alt,Ret,  & All \\
	    &     &     & Dis  & Dis & I,Dis & I,Dis  & I,Dis  & I,Dis & S,F,Ex,   & \\
	    &     &     &      &     &       &        &        &       & Exi,I,Dis & \vspace{-0.6em}\\ \hline\hline%\cmidrule{1-11}

\end{tabular}

\label{tab:composition_table_10}
\end{table}

\section{Implementation}
\label{sec:impl}

The nature of qualitative calculi such as TC-6 and TC-10 and the associated composition tables allows for a straightforward transformation to any formalism that can capture boolean satisfiability. Following the work of Brenton et al.~\citeyear{Brenton2016}, we propose the implementation of TC-6 and TC-10 as ASP programs, such that each answer set corresponds to a valid configuration, i.e., an assignment of relations to a set of trajectories that conforms to the corresponding composition table. Our starting point is the best-performing ASP encoding in Brenton et al.~\citeyear{Brenton2016}, \textit{COI7}, which stands for Choice rule with One predicate per pair and Integrity constraints for composition table entries with more than 7 relations. This encoding has the following characteristics: (1) search space (the possible relations $r_1,...,r_n$ in the calculus) is represented as a choice rule of the form $\{ r_1(X,Z) ; ... ; r_n(X,Z) \} = 1 \leftarrow traj(X), traj(Z), X!=Z.$, with (2) inverse relations are modelled using a single predicate ($ex(x,y)$ and $ex(y,x)$ instead of $exi(x,y)$); (3) the composition table entries are encoded either as disjunctive rules of the form $r_1(X,Z) | ... | r_n(X,Z) \leftarrow r_a(X,Y), r_b(Y,Z).$, when the entry contains up to 7 relations, or as integrity constraints of the form $\leftarrow r_i(X,Z), r_a(X,Y), r_b(Y,Z)$, for all $r_i$ that are \textit{not} part of the composition table entry; (4) there is an additional integrity constraint for each pair of relations, to model the fact that they are pairwise disjoint.

To improve on COI7, we rely on the specific characteristics of TC-6 and TC-10. First, there are no inverse relations in TC-6 and only one pair in TC-10 (Ex and Exi), while all of the relations apart from the inverse pair are symmetric. This means that there is no significant benefit of using one predicate per inverse pair rather than two. At the same time, using two predicates allows us to apply the antisymmetric optimisation of Brenton et al.~\citeyear{Brenton2016} which should significantly improve performance since it reduces the possible relation pairs by half. For TC-6, antisymmetric optimisation amounts to replacing $X!=Z$ with $X<Z$ in the choice rule, while for TC-10, we have to also use different predicates for Ex and Exi and include two rules that generate an Exi relation for each known Ex relation and vice-versa. Note that in order for antisymmetric optimisation to be applied consistently, we need to ensure that all relations included in the predefined rules and facts have a first operand that is arithmetically (or lexicographically) before the second one.

Secondly, we removed all integrity constraints for pairwise disjointness, since the use of a choice rule accounts for that. Finally, we opted to use integrity constraints to model all composition entry tables and not just the ones that contain more than 7 relations. In this manner, there is only one simplified integrity constraint per each table entry, of the following form: $\leftarrow not \ r_1(X,Z), ..., not \ r_n(X,Z), r_a(X,Y), r_b(Y,Z)$ (negating only the relations contained in the entry). This decision is due to the fact that the disjunctive rules used in COI7 for entries with 7 or less relations can lead to an ASP program with head cycles. Deciding the existence of an answer set for such programs is $\Sigma^P_2$-complete in general, while the same problem is NP-complete in absence of head cycles~\cite{Eiter1997}. Answer set existence is guaranteed to be in NP for all encodings proposed in this paper. Following the naming convention of Brenton et al.~\citeyear{Brenton2016}, we name this encoding CTSA (Choice rule with Two predicates per pair, \textit{Simplified} integrity constraints and Antisymmetric optimisation).

A further improvement is possible if we reduce the workload of the ASP grounder by making sure that no grounding is performed for pairs of trajectories where the relation is known (i.e., is a fact). To achieve this, we distinguish between a known relation and a derived one. Instead of using an $r_i(X,Y)$ in both cases, as done in COI7 and CTSA, we use instead a predicate $fact(r_i,X,Y)$ for known relations and include the conjunct $\#count\{R : fact(R,X,Y)\} = 0$ at the end of the choice rule, to make sure that it is only applicable when the relation between two trajectories is not known. Clearly, the more relations are known, the more significant this improvement will be. We name this encoding CTSA2.

\subsection{Generalised ASP Encoding for Qualitative Calculi}
\label{sec:gen}

Based on the knowledge gained by the various encodings presented so far, we now follow a systematic approach that results in a generalised ASP encoding, applicable to any qualitative calculus whose definition includes a composition table.

The first step is to represent the domain: any element $x$ that is modelled by the calculus (e.g., trajectories for TC-6 and TC-10), is expressed using a fact $element(x)$. To represent the base relations of the calculus, instead of a predicate for each (since each calculus has its own relations), we define a predicate of arity 1, named \textit{relation} which is instantiated for each relation included in the calculus (e.g., $relation(eq)$, $relation(alt)$, and so on).

Then, to encode composition table entries, we first define a \textit{table} predicate with three arguments: $table(row\_relation,column\_relation,valid\_relation)$. For each cell in the composition table, we instantiate as many table predicates as the possible relations included in that cell. For example, to represent the (alt, i) cell in the composition table of TC-6, we need the following: $table(alt, i, i)$ and $table(alt, i, dis)$. These steps collectively store the knowledge contained in a composition table as a set of facts.

To implement qualitative reasoning based on the encoded table, we adopt the decision made for CTSA, to use integrity constraints to enforce every composition table entry. Since all table entries are represented by a table predicate, we need only include one generalised integrity constraint, using a conditional literal: $\leftarrow true(X,R_1,Y), true(Y,R_2,Z), not \ true(X,R_\mathit{out},Z) : table(R_1,R_2,R_\mathit{out}).$ Predicate \textit{true(X,R,Y)} states that relation R holds for trajectory pair (X,Y). Conditional literals are convenient here since we need to represent conjunctions involving a variable number of literals. This integrity constraint essentially states that, if $R_1$ and $R_2$ hold for $(X,Y)$ and $(Y,Z)$, respectively, then there must be at least one table entry $(R_1,R_2, R_\mathit{out})$ such that $R_{out}$ holds for $(X,Z)$.

To enforce the fact that the set of base relations is jointly exhaustive and pairwise disjoint, we include the choice rule $\{true(X,R,Y) : relation(R)\} = 1 \leftarrow element(X), element(Y), X != Y$. Additionally, to specify that one of the base relations of the calculus represents equality, we include a special rule of the form: $true(X,eq,X) \leftarrow element(X)$.

Finally, to represent input, we introduce a predicate of arity 3 named \textit{possible}: $possible(X,R,Y)$ states that the pair $(X,Y)$ is involved in a constraint, and $R$ is a possible relation.
An additional integrity constraint is necessary to relate instantiations of possible and true predicates: $\leftarrow possible(X,_,Y), not\ true(X,R,Y) : possible(X,R,Y).$

The aforementioned decisions that were necessary for this generalised encoding (named GEN) to be applicable to any qualitative calculus inevitably result in decreased performance, as analysed in Section~\ref{sec:eval}. The intention of the systematic approach described in this section is to provide a means of quickly producing a universally applicable ASP encoding, whose applicability can then be restricted to improve performance through calculus-specific optimisations.

\section{Evaluation}
\label{sec:eval}
To compare the efficiency of the proposed encodings, we conducted an experimental evaluation using the T-Drive dataset~\cite{Yuan2013}, which contains trajectories generated by 10,357 taxis in a week with about 15 million points and 9 million kilometres length of the total distance in Beijing. We first applied a data cleaning process to remove records missing latitude/longitude values or containing significantly infeasible values, effectively sieving out incomplete and noisy data. After this process, trajectories generated by 1000 cars are sampled from the clean dataset.
%Taking a sample trajectory from a car as the example, the trajectory is described as a set of lat-long literal, which is defined by $t_i = \{(lat_1,lng_1),(lat_2,lng_2),...,(lat_n,lng_n)\}$ and $T = \{t_1, t_2,...,t_m\} \text{ subject to } T \in R$, where $lat$ and $lng$ are the latitude and longitude values, respectively, $t_i$ is the trajectory of the $i^{th}$ car, $T$ is the set of $m$ sampled trajectories, and $R$ indicate the region of Beijing city. For every individual in set $T$, $t_i$ is mapped into a square region bounded by the combination of the $\max_{j=1,2,...,n}(lat_j)$, $\max_{j=1,2,...,n}(lng_j)$, $\min_{j=1,2,...,n}(lat_j)$, and $\min_{j=1,2,...,n}(lng_j)$ with a comparative resolution composed of $100\times200$.
Each sampled trajectory is mapped from a sequence of latitude/longitude pairs into a sequence of corresponding regions within the boundaries formed by the maximum/minimum latitudes and longitudes of the targeted car's trajectory records, with the map of Beijing city having a comparative resolution of $100\times200$ regions. The choice of this resolution is informed by both the density of the latitude/longitude values within the Beijing city area and the trajectory records in the spatio-temporal domain. The resulting 1000 trajectories have a mean length of 282 regions with a standard deviation of 33.27. The dataset and code used in the experiments and all encodings are available at github.com/gmparg/ICLP2018. Partial encodings are also included in~\ref{sec:encod}.

We ran two different types of experiments for both TC-6 and TC-10. In the first type we kept the number of known relations fixed to one per distinct trajectory and varied the number of trajectories from 10 to 250. In the second type, we used a fixed number of 50 trajectories and varied the percentage of known relations from 6\% to 100\%. In both cases we used the ASP system clingo version 5.2.0~\cite{Gebser2016}, tasking it to derive a single solution. Time and memory values were calculated using pyrunlim (available at github.com/alviano/python/tree/master/pyrunlim). All experiments were performed on a Debian Linux server with an Intel\textregistered~Xeon\textregistered~X3430 CPU at 2.4GHz, with 16 GB RAM.

\begin{figure}
%\begin{subfigure}{.498\linewidth}
\centering
\subfloat{
\begin{tikzpicture}[node distance = 2cm, scale=0.8, transform shape]
\begin{axis}[xlabel= Trajectories,ylabel= CPU time (s),legend pos= north west]
    \addplot+[color=color1,mark=otimes, mark size=2,error bars/.cd, y dir=both,y explicit] coordinates {
    (10,0.108)
    (20,1.020)
    (30,3.620)
    (40,8.970)
    (50,18.300)
    (60,32.160)
    (70,53.600)
    (80,81.290)
    (90,121.550)
    (100,172.070)
    (110,238.000)
    };
    \addplot+[color=color2,mark=x, mark size=2,error bars/.cd, y dir=both,y explicit] coordinates {
    (10,0.016)
    (20,0.110)
    (30,0.340)
    (40,1.030)
    (50,2.100)
    (60,4.060)
    (70,6.390)
    (80,9.480)
    (90,13.620)
    (100,19.330)
    (110,24.960)
    (120,34.120)
    (130,45.300)
    (140,55.520)
    (150,69.870)
    (160,84.170)
    (170,103.620)
    (180,123.110)
    (190,155.010)
    (200,179.700)
    (210,203.370)
    (220,235.360)
    (230,271.610)
    (240,326.790)
    (250,372.550)
    };
    \addplot+[color=color3,mark=oplus, mark size=2, error bars/.cd, y dir=both,y explicit] coordinates {
    (10,0.016)
    (20,0.108)
    (30,0.424)
    (40,1.020)
    (50,2.100)
    (60,3.840)
    (70,5.870)
    (80,9.490)
    (90,13.620)
    (100,19.830)
    (110,25.980)
    (120,36.160)
    (130,46.380)
    (140,57.590)
    (150,74.980)
    (160,91.390)
    (170,112.890)
    (180,134.340)
    (190,168.280)
    (200,191.980)
    (210,227.990)
    (220,266.010)
    (230,307.380)
    (240,367.070)
    (250,414.610)
    };
    \addplot+[color=color4,mark=+, mark size=2, error bars/.cd, y dir=both,y explicit] coordinates {
    (10,0.140)
    (20,1.232)
    (30,4.368)
    (40,10.912)
    (50,22.016)
    (60,39.276)
    (70,64.156)
    (80,98.508)
    (90,140.988)
    (100,196.776)
    (110,266.244)
    (120,352)
    (130,458.848)
    (140,574.176)
    };
    \legend{COI7, CTSA, CTSA2, GEN}
    \end{axis}
\end{tikzpicture}
\label{fig:eval-tc6-d1time}
}\quad
%\end{subfigure}
%\begin{subfigure}{.498\linewidth}
\subfloat{
%\centering
\begin{tikzpicture}[node distance = 2cm, scale=0.8, transform shape]
\begin{axis}[xlabel= Trajectories,ylabel= Memory (GB),legend pos= north west]
    \addplot+[color=color1,mark=otimes, mark size=2,error bars/.cd, y dir=both,y explicit] coordinates {
    (10,0)
    (20,0.086)
    (30,0.290)
    (40,0.668)
    (50,1.326)
    (60,2.299)
    (70,3.689)
    (80,5.486)
    (90,7.852)
    (100,10.842)
    (110,14.596)
    };
    \addplot+[color=color2,mark=x, mark size=2,error bars/.cd, y dir=both,y explicit] coordinates {
    (10,0)
    (20,0.0141)
    (30,0.028)
    (40,0.066)
    (50,0.123)
    (60,0.200)
    (70,0.316)
    (80,0.401)
    (90,0.669)
    (100,0.944)
    (110,1.212)
    (120,1.635)
    (130,2.039)
    (140,2.511)
    (150,3.194)
    (160,3.808)
    (170,4.519)
    (180,5.322)
    (190,6.488)
    (200,7.471)
    (210,8.561)
    (220,9.843)
    (230,11.055)
    (240,13.041)
    (250,14.601)
    };
    \addplot+[color=color3,mark=oplus, mark size=2, error bars/.cd, y dir=both,y explicit] coordinates {
    (10,0)
    (20,0)
    (30,0.026)
    (40,0.062)
    (50,0.113)
    (60,0.195)
    (70,0.307)
    (80,0.465)
    (90,0.652)
    (100,0.928)
    (110,1.202)
    (120,1.608)
    (130,2.013)
    (140,2.474)
    (150,3.148)
    (160,3.757)
    (170,4.460)
    (180,5.257)
    (190,6.420)
    (200,7.396)
    (210,8.476)
    (220,9.662)
    (230,10.952)
    (240,12.927)
    (250,14.477)
    };
    \addplot+[color=color4,mark=+, mark size=2, error bars/.cd, y dir=both,y explicit] coordinates {
    (10,0.014)
    (20,0.054)
    (30,0.160)
    (40,0.374)
    (50,0.664)
    (60,1.145)
    (70,1.870)
    (80,2.775)
    (90,3.762)
    (100,5.195)
    (110,6.976)
    (120,8.738)
    (130,11.499)
    (140,14.449)
    };
    \legend{COI7, CTSA, CTSA2, GEN}
    \end{axis}
\end{tikzpicture}
\label{fig:eval-tc6-d1mem}
%\end{subfigure}
}
\caption{TC-6: Finding a consistent configuration when only one relation per trajectory is known.}
\label{fig:eval-tc6-d1}
\end{figure}
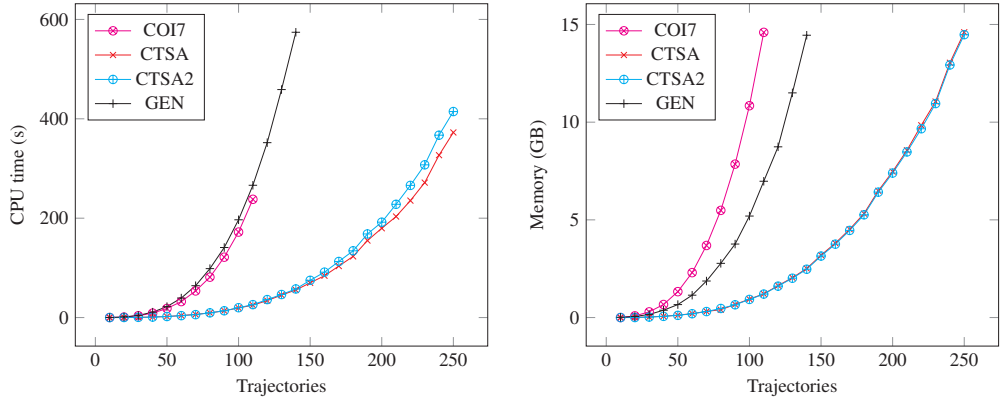

The results of the first TC-6 experiment are shown in Figure~\ref{fig:eval-tc6-d1}. The improved encodings allow the ASP solver to derive a consistent configuration for up to 250 trajectories, when the COI7 encoding can only yield results for up to 110 trajectories, before running out of memory. The GEN encoding is slightly slower than COI7 but memory consumption is slightly better, allowing results for up to 140 trajectories. CTSA and CTSA2 are significantly faster, achieving execution times one order of magnitude less than COI7 and GEN. CTSA2 performs similarly to CTSA in terms of memory consumption, while it is slightly slower for more than 100 trajectories. This is expected since in this experiment the input contains only a single relation per distinct pair of trajectories, hence the improvements of CTSA2 have an insignificant positive effect to memory consumption and the auxiliary predicate it employs adds unnecessary complexity which has a slightly negative effect on execution time for larger trajectory sets.

\begin{figure}
\centering
%\begin{subfigure}{.498\linewidth}
\subfloat{
\begin{tikzpicture}[node distance = 2cm, scale=0.78, transform shape]
\begin{axis}[xlabel= Number of known relations per trajectory,ylabel= CPU time (s),legend style={at={(0.32,0.54)}}]
    \addplot+[color=color1,mark=otimes, mark size=2,error bars/.cd, y dir=both,y explicit] coordinates {
    (3,19.740)
    (5,18.840)
    (10,16.750)
    (15,15.210)
    (20,13.640)
    (25,12.090)
    (30,11.570)
    (35,11.560)
    (40,10.530)
    (45,10.010)
    (50,8.980)
    };
    \addplot+[color=color2,mark=x, mark size=2,error bars/.cd, y dir=both,y explicit] coordinates {
    (3,2.090)
    (5,2.100)
    (10,2.100)
    (15,2.090)
    (20,2.100)
    (25,1.880)
    (30,1.880)
    (35,1.880)
    (40,1.450)
    (45,1.230)
    (50,1.020)
    };
    \addplot+[color=color3,mark=oplus, mark size=2, error bars/.cd, y dir=both,y explicit] coordinates {
    (3,1.890)
    (5,1.660)
    (10,1.230)
    (15,0.910)
    (20,0.680)
    (25,0.220)
    (30,0.220)
    (35,0.220)
    (40,0.110)
    (45,0.108)
    (50,0.032)
    };
    \addplot+[color=color4,mark=+, mark size=2, error bars/.cd, y dir=both,y explicit] coordinates {
    (3,22.240)
    (5,22.452)
    (10,22.692)
    (15,22.664)
    (20,22.596)
    (25,22.280)
    (30,22.184)
    (35,22.048)
    (40,21.696)
    (45,21.256)
    (50,20.800)
    };
    \legend{COI7, CTSA, CTSA2, GEN}
    \end{axis}
\end{tikzpicture}
\label{fig:eval-tc6-50time}
%\end{subfigure}
}\quad
\subfloat{
%\begin{subfigure}{.498\linewidth}
%\centering
\begin{tikzpicture}[node distance = 2cm, scale=0.78, transform shape]
\begin{axis}[xlabel= Number of known relations per trajectory,ylabel= Memory (MB),legend pos= north east]
    \addplot+[color=color1,mark=otimes, mark size=2,error bars/.cd, y dir=both,y explicit] coordinates {
    (3,1262.2)
    (5,1152.2)
    (10,900.3)
    (15,700.6)
    (20,585.4)
    (25,411.5)
    (30,408.2)
    (35,398.3)
    (40,382.8)
    (45,359.1)
    (50,336.7)
    };
    \addplot+[color=color2,mark=x, mark size=2,error bars/.cd, y dir=both,y explicit] coordinates {
    (3,125.3)
    (5,124.3)
    (10,117.3)
    (15,107.4)
    (20,108.3)
    (25,105.4)
    (30,103.6)
    (35,97.7)
    (40,84.6)
    (45,72.2)
    (50,60)
    };
    \addplot+[color=color3,mark=oplus, mark size=2, error bars/.cd, y dir=both,y explicit] coordinates {
    (3,106.2)
    (5,101)
    (10,72.9)
    (15,50.6)
    (20,39.4)
    (25,20.9)
    (30,20.5)
    (35,20.3)
    (40,13.4)
    (45,0)
    (50,0)
    };
    \addplot+[color=color4,mark=+, mark size=2, error bars/.cd, y dir=both,y explicit] coordinates {
    (3,676.6)
    (5,666.4)
    (10,666.3)
    (15,681.2)
    (20,672)
    (25,658.1)
    (30,655.5)
    (35,650.6)
    (40,647.7)
    (45,635.3)
    (50,623.2)
    };
    \legend{COI7, CTSA, CTSA2, GEN}
    \end{axis}
\end{tikzpicture}
\label{fig:eval-tc6-50mem}
%\end{subfigure}
}
\caption{TC-6: Determining consistency for 50 trajectories.}
\label{fig:eval-tc6-50}
\end{figure}
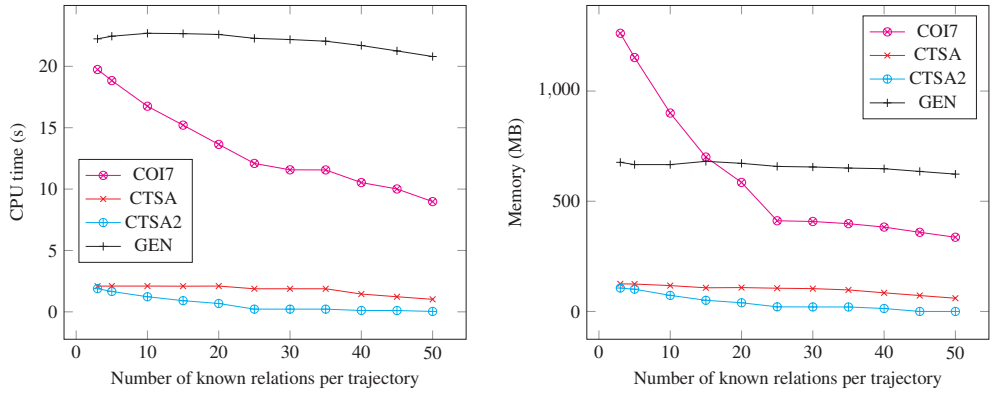

The benefits of CTSA2 are better illustrated in the second TC-6 experiment (Figure~\ref{fig:eval-tc6-50}). As the number of known relations per trajectory increases, both execution time and memory decrease faster for CTSA2 than CTSA and GEN, up to an 88\% decrease in execution time and a 80\% decrease in memory consumption for the case of knowing 25 (out of 50 possible) relations per trajectory. Afterwards, the improvement of CTSA2 is less pronounced, since the task becomes increasingly easier as less relations are unknown. COI7 exhibits a similar behaviour: execution time and memory rapidly decrease up to the case of knowing 25 relations per trajectory, after which the decrease is more gradual. With regard to GEN, while the variation is not significant, we do observe a slight increase (up to roughly 3\%) in execution time at low percentages of known relations and a similar decrease afterwards.% This behaviour is similar to what~\cite{Cheeseman1991} called phase transition.

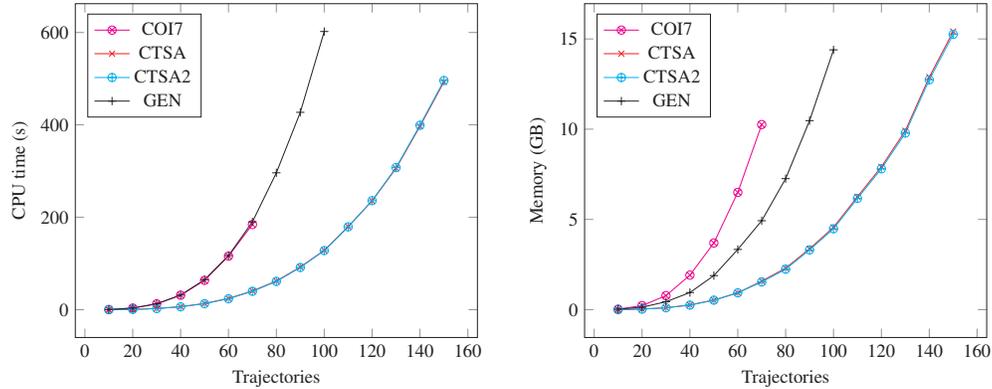
\begin{figure}
\centering
%\begin{subfigure}{.498\linewidth}
\subfloat{
\begin{tikzpicture}[node distance = 2cm, scale=0.78, transform shape]
\begin{axis}[xlabel= Trajectories,ylabel= CPU time (s),legend pos= north west]
    \addplot+[color=color1,mark=otimes, mark size=2,error bars/.cd, y dir=both,y explicit] coordinates {
    (10,0.360)
    (20,3.484)
    (30,12.708)
    (40,31.384)
    (50,63.432)
    (60,115.996)
    (70,184.356)
    };
    \addplot+[color=color2,mark=x, mark size=2,error bars/.cd, y dir=both,y explicit] coordinates {
    (10,0.072)
    (20,0.652)
    (30,2.564)
    (40,6.564)
    (50,13.436)
    (60,24.400)
    (70,40.480)
    (80,62.044)
    (90,92.356)
    (100,128.404)
    (110,179.788)
    (120,235.504)
    (130,306.044)
    (140,396.896)
    (150,492.976)
    };
    \addplot+[color=color3,mark=oplus, mark size=2, error bars/.cd, y dir=both,y explicit] coordinates {
    (10,0.06)
    (20,0.588)
    (30,2.404)
    (40,6.284)
    (50,12.996)
    (60,23.776)
    (70,39.796)
    (80,61.224)
    (90,91.420)
    (100,127.944)
    (110,179.356)
    (120,235.976)
    (130,307.564)
    (140,399.212)
    (150,495.996)
    };
     \addplot+[color=color4,mark=+, mark size=2, error bars/.cd, y dir=both,y explicit] coordinates {
    (10,0.384)
    (20,3.532)
    (30,12.748)
    (40,31.680)
    (50,64.916)
    (60,116.944)
    (70,190.184)
    (80,296.292)
    (90,427.196)
    (100,602.076)
    };
    \legend{COI7, CTSA, CTSA2, GEN}
    \end{axis}
\end{tikzpicture}
\label{fig:eval-tc10-d1time}
%\end{subfigure}
}\quad
\subfloat{
%\begin{subfigure}{.498\linewidth}
%\centering
\begin{tikzpicture}[node distance = 2cm, scale=0.78, transform shape]
\begin{axis}[xlabel= Trajectories,ylabel= Memory (GB),legend pos= north west]
    \addplot+[color=color1,mark=otimes, mark size=2,error bars/.cd, y dir=both,y explicit] coordinates {
    (10,0.028)
    (20,0.223)
    (30,0.779)
    (40,1.914)
    (50,3.695)
    (60,6.498)
    (70,10.255)
    };
    \addplot+[color=color2,mark=x, mark size=2,error bars/.cd, y dir=both,y explicit] coordinates {
    (10,0)
    (20,0.032)
    (30,0.105)
    (40,0.263)
    (50,0.527)
    (60,0.934)
    (70,1.580)
    (80,2.301)
    (90,3.373)
    (100,4.562)
    (110,6.258)
    (120,7.916)
    (130,9.915)
    (140,12.885)
    (150,15.411)
    };
    \addplot+[color=color3,mark=oplus, mark size=2, error bars/.cd, y dir=both,y explicit] coordinates {
    (10,0)
    (20,0.030)
    (30,0.105)
    (40,0.251)
    (50,0.526)
    (60,0.933)
    (70,1.537)
    (80,2.242)
    (90,3.310)
    (100,4.482)
    (110,6.166)
    (120,7.804)
    (130,9.782)
    (140,12.735)
    (150,15.252)
    };
    \addplot+[color=color4,mark=+, mark size=2, error bars/.cd, y dir=both,y explicit] coordinates {
    (10,0.025)
    (20,0.138)
    (30,0.436)
    (40,0.949)
    (50,1.887)
    (60,3.338)
    (70,4.929)
    (80,7.260)
    (90,10.474)
    (100,14.395)
    };
    \legend{COI7, CTSA, CTSA2, GEN}
    \end{axis}
\end{tikzpicture}
\label{fig:eval-tc10-d1mem}
%\end{subfigure}
}
\caption{TC-10: Finding a consistent configuration when only one relation per trajectory is known.}
\label{fig:eval-tc10-d1}
\end{figure}

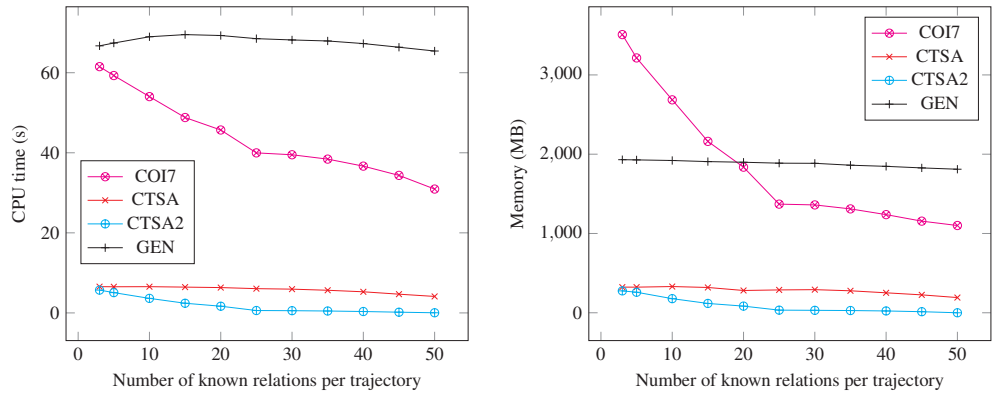
\begin{figure}
\centering
%\begin{subfigure}{.498\linewidth}
\subfloat{
\begin{tikzpicture}[node distance = 2cm, scale=0.78, transform shape]
\begin{axis}[xlabel= Number of known relations per trajectory,ylabel= CPU time (s),legend style={at={(0.32,0.54)}}]
    \addplot+[color=color1,mark=otimes, mark size=2,error bars/.cd, y dir=both,y explicit] coordinates {
    (3,61.516)
    (5,59.292)
    (10,54.028)
    (15,48.804)
    (20,45.692)
    (25,39.976)
    (30,39.504)
    (35,38.416)
    (40,36.664)
    (45,34.352)
    (50,30.936)
    };
    \addplot+[color=color2,mark=x, mark size=2,error bars/.cd, y dir=both,y explicit] coordinates {
    (3,6.560)
    (5,6.532)
    (10,6.554)
    (15,6.444)
    (20,6.328)
    (25,6.060)
    (30,5.936)
    (35,5.656)
    (40,5.276)
    (45,4.680)
    (50,4.116)
    };
    \addplot+[color=color3,mark=oplus, mark size=2, error bars/.cd, y dir=both,y explicit] coordinates {
    (3,5.700)
    (5,5.068)
    (10,3.632)
    (15,2.416)
    (20,1.672)
    (25,0.596)
    (30,0.557)
    (35,0.476)
    (40,0.348)
    (45,0.180)
    (50,0.036)
    };
    \addplot+[color=color4,mark=+, mark size=2, error bars/.cd, y dir=both,y explicit] coordinates {
    (3,66.712)
    (5,67.412)
    (10,68.980)
    (15,69.520)
    (20,69.292)
    (25,68.508)
    (30,68.184)
    (35,67.928)
    (40,67.280)
    (45,66.364)
    (50,65.420)
    };
    \legend{COI7, CTSA, CTSA2, GEN}
    \end{axis}
\end{tikzpicture}
\label{fig:eval-tc10-50time}
%\end{subfigure}
}\quad
\subfloat{
%\begin{subfigure}{.498\linewidth}
%\centering
\begin{tikzpicture}[node distance = 2cm, scale=0.78, transform shape]
\begin{axis}[xlabel= Number of known relations per trajectory,ylabel= Memory (MB),legend pos= north east]
    \addplot+[color=color1,mark=otimes, mark size=2,error bars/.cd, y dir=both,y explicit] coordinates {
    (3,3508.5)
    (5,3214.3)
    (10,2685.6)
    (15,2161.7)
    (20,1836.4)
    (25,1370.2)
    (30,1360.3)
    (35,1310.1)
    (40,1237.9)
    (45,1155.7)
    (50,1100.6)
    };
    \addplot+[color=color2,mark=x, mark size=2,error bars/.cd, y dir=both,y explicit] coordinates {
    (3,323.4)
    (5,323.1)
    (10,330.8)
    (15,319)
    (20,281.1)
    (25,288.2)
    (30,291.5)
    (35,277.2)
    (40,251.7)
    (45,225.7)
    (50,191.8)
    };
    \addplot+[color=color3,mark=oplus, mark size=2, error bars/.cd, y dir=both,y explicit] coordinates {
    (3,276)
    (5,259)
    (10,177.9)
    (15,117.6)
    (20,83.9)
    (25,33.2)
    (30,30.8)
    (35,28.1)
    (40,22.9)
    (45,13.5)
    (50,0)
    };
\addplot+[color=color4,mark=+, mark size=2, error bars/.cd, y dir=both,y explicit] coordinates {
    (3,1930.8)
    (5,1928.2)
    (10,1920.6)
    (15,1906.7)
    (20,1898.5)
    (25,1886.8)
    (30,1885.1)
    (35,1861.2)
    (40,1846.8)
    (45,1827.1)
    (50,1810.5)
    };
    \legend{COI7, CTSA, CTSA2, GEN}
    \end{axis}
\end{tikzpicture}
\label{fig:eval-tc10-50mem}
%\end{subfigure}
}
\caption{TC-10: Determining consistency for 50 trajectories.}
\label{fig:eval-tc10-50}
\end{figure}

The results for the first TC-10 experiment are shown in Figure~\ref{fig:eval-tc10-d1}; as expected, the behaviour is similar to TC-6, with increased costs in time and memory due to the increased complexity. The CTSA and CTSA2 encodings manage to yield results for up to 150 trajectories, while COI7 stops at 70 and GEN stops at 100, both due to running out of memory. The slight variation in execution time between CTSA and CTSA2 that was evident in TC-6 is not noticeable in TC-10 since the higher complexity of the additional 4 relations does not allow executions for more than 150 trajectories, where the variation would be more evident. The results are also similar in the second TC-10 experiment, shown in Figure~\ref{fig:eval-tc10-50}, where the improvement of CTSA2 over CTSA is again increasingly evident as the percentage of known relations increases and peaks at the case of knowing 25 out of 50 relations per trajectory. Finally, GEN again shows little variation, with a slight increase of up to 6.5\%, followed by a slight decrease.

%GEN again exhibits a "phase-transition" behaviour, with an increase region up to 50\% and a decrease region afterwards.

\section{Conclusion and Future Work}
\label{sec:concl}

In this paper, we proposed two qualitative calculi for trajectories defined as sequences of non-overlapping regions on partitioned maps. We then detailed the implementation of these calculi using ASP, applying several optimisations that allow the best performing implementation to scale up to an input of 250 trajectories, with each trajectory consisting of more than 50 regions on average. We also provided a generalised ASP encoding that is applicable to any qualitative calculus that includes a composition table in its definition. We expect the calculi and their implementations to prove useful in applications relying on reasoning over large trajectory databases, as well as in related attempts to implement qualitative reasoning using logic programming.

Future research directions include: (a) exploring a third, more restrictive trajectory calculus where no cycles are allowed, meaning that any region may be visited only once within a given trajectory; (b) determining whether the improvements achieved in the implementation of TC-6 and TC-10 as ASP programs can be carried over to ASP implementations of other qualitative calculi such as the Region Connection Calculus or Allen's interval algebra; (c) extending the calculus with a temporal dimension by combining it with Allen's interval algebra; (d) investigating the trade-off between performance of the generic encoding and its range of applicability.

\appendix

%\mbox{}\vfill\eject

\section{Proofs}
\label{sec:proofs}

We start by making the link between models of qualitative calculi and answer sets of GEN explicit, that is, by providing a mapping between them.

\begin{lemma}\label{lem:qc}
GEN encodes QC model existence.
\end{lemma}
\begin{proof}
Let $(\mathcal{E},\mathcal{R},c,\mathcal{C})$ be an instance of QC model existence, and $\Pi$ be its GEN encoding.
Let $\nu : \mathcal{E} \times \mathcal{E} \to \mathcal{R}$ be an assignment.
We build interpretation $I$ containing the following atoms:
for all $x \in \mathcal{E}$, $\mathit{element}(x)$;
for all $r \in \mathcal{R}$, $\mathit{relation}(r)$;
for all $r,r',r'' \in \mathcal{R}$ such that $r'' \in c(r,r')$, $\mathit{table}(r,r',r'')$;
for all $x \in \mathcal{E}$, $\mathit{true}(x,\mathit{eq},x)$;
for all $(x,y) \in R$ in $\mathcal{C}$, and for all $r \in R$, $\mathit{possible}(x,r,y)$;
for all $x,y \in E$, if $\nu(x,y) = r$, then $\mathit{true}(x,r,y)$.
It holds that $\nu$ is a model of $(\mathcal{E},\mathcal{R},c,\mathcal{C})$ if and only if $I$ is an answer set of $\Pi$.
\hfill
\end{proof}

Membership in NP is a consequence of the fact that coherence of GEN can be checked by an NP procedure.

\ThmNP*
\begin{proof}
The claim follows from Lemma~\ref{lem:qc} and by the fact that answer set existence of ASP programs containing only choice rules and integrity constraints belongs to NP  \cite{DBLP:conf/slp/MarekT89,DBLP:journals/jlp/CadoliS93}.
\hfill
\end{proof}

\section{Partial ASP Encodings}
\label{sec:encod}

\begin{figure} [!h]
\footnotesize
\figrule
\begin{verbatim}
{s(X,Y); f(X,Y); alt(X,Y); i(X,Y); eq(X,Y); dis(X,Y)}=1 :- traj(X), traj(Y), X!=Y.
eq(X,X) :- traj(X).
eq(Z,X) :- eq(Y,X), eq(Z,Y).                 f(X,Z) :- f(X,Y), eq(Z,Y).
alt(Z,X) :- eq(Y,X), alt(Z,Y).               f(X,Z) :- f(X,Y), alt(Z,Y).
s(X,Z) :- eq(Y,X), s(Y,Z).                   i(X,Z) | dis(Z,X) :- f(X,Y), s(Y,Z).
f(X,Z) :- eq(Y,X), f(Y,Z).                   eq(Z,X) | alt(Z,X) | f(X,Z) :- f(X,Y), f(Y,Z).
i(X,Z) :- eq(Y,X), i(Y,Z).                   s(X,Z) | i(X,Z) | dis(Z,X) :- f(X,Y), i(Y,Z).
dis(Z,X) :- eq(Y,X), dis(Z,Y).               s(X,Z) | i(X,Z) | dis(Z,X) :- f(X,Y), dis(Z,Y).
:- eq(Z,X), s(X,Z).                          :- f(X,Z), s(X,Z).
:- eq(Z,X), f(X,Z).                          :- f(X,Z), alt(Z,X).
:- eq(Z,X), alt(Z,X).                        :- f(X,Z), i(X,Z).
:- eq(Z,X), i(X,Z).                          :- f(X,Z), eq(Z,X).
:- eq(Z,X), dis(Z,X).                        :- f(X,Z), dis(Z,X).
\end{verbatim}
\caption{Partial COI7 encoding for TC-6.}\label{fig:tc6coi7}
\figrule
\end{figure}

\begin{figure} [p]
\footnotesize
\figrule
\begin{verbatim}
{s(X,Y); f(X,Y); alt(X,Y); i(X,Y); eq(X,Y); dis(X,Y)}=1 :- traj(X), traj(Y), X<Y.
eq(X,X) :- traj(X).
:- eq(X,Y), eq(Y,Z), not eq(X,Z).   :- f(X,Y), eq(Y,Z), not f(X,Z).
:- eq(X,Y), alt(Y,Z), not alt(X,Z). :- f(X,Y), alt(Y,Z), not f(X,Z).
:- eq(X,Y), s(Y,Z), not s(X,Z).     :- f(X,Y), s(Y,Z), not i(X,Z), not dis(X,Z).
:- eq(X,Y), f(Y,Z), not f(X,Z).     :- f(X,Y), f(Y,Z), not eq(X,Z), not alt(X,Z), not f(X,Z).
:- eq(X,Y), i(Y,Z), not i(X,Z).     :- f(X,Y), i(Y,Z), not s(X,Z), not i(X,Z), not dis(X,Z).
:- eq(X,Y), dis(Y,Z), not dis(X,Z). :- f(X,Y), dis(Y,Z), not s(X,Z), not i(X,Z), not dis(X,Z).
\end{verbatim}
\caption{Partial CTSA encoding for TC-6.}\label{fig:tc6ctsa}
\figrule
\end{figure}

\begin{figure} [p]
\footnotesize
%\figrule
\begin{verbatim}
{s(X,Y); f(X,Y); alt(X,Y); i(X,Y); eq(X,Y); dis(X,Y)}=1 :- traj(X), traj(Y), X<Y,
 #count{R : fact(R,X,Y)} = 0.
eq(X,X) :- traj(X).
:- eq(X,Y), eq(Y,Z), not eq(X,Z).   :- f(X,Y), eq(Y,Z), not f(X,Z).
:- eq(X,Y), alt(Y,Z), not alt(X,Z). :- f(X,Y), alt(Y,Z), not f(X,Z).
:- eq(X,Y), s(Y,Z), not s(X,Z).     :- f(X,Y), s(Y,Z), not i(X,Z), not dis(X,Z).
:- eq(X,Y), f(Y,Z), not f(X,Z).     :- f(X,Y), f(Y,Z), not eq(X,Z), not alt(X,Z), not f(X,Z).
:- eq(X,Y), i(Y,Z), not i(X,Z).     :- f(X,Y), i(Y,Z), not s(X,Z), not i(X,Z), not dis(X,Z).
:- eq(X,Y), dis(Y,Z), not dis(X,Z). :- f(X,Y), dis(Y,Z), not s(X,Z), not i(X,Z), not dis(X,Z).
eq(X,Y) :- fact(eq,X,Y).
f(X,Y) :- fact(f,X,Y).
\end{verbatim}
\caption{Partial CTSA2 encoding for TC-6.}\label{fig:tc6ctsa2}
\figrule
\end{figure}

\begin{figure} [p]
\footnotesize
%\figrule
\begin{verbatim}
{true(X,R,Y) : relation(R)} = 1 :- element(X); element(Y); X != Y.
true(X,eq,X) :- element(X).
:- true(X,R1,Y); true(Y,R2,Z); not true(X,Rout,Z) : table(R1,R2,Rout).
:- possible(X,_,Y); not true(X,R,Y) : possible(X,R,Y).
relation(eq; alt; s; f; i; dis).
table(eq, eq, (eq)).                    table(f, eq, (f)).
table(eq, alt, (alt)).                  table(f, alt, (f)).
table(eq, s, (s)).                      table(f, s, (i;dis)).
table(eq, f, (f)).                      table(f, f, (eq;alt;f)).
table(eq, i, (i)).                      table(f, i, (s;i;dis)).
table(eq, dis, (dis)).                  table(f, dis, (s;i;dis)).
\end{verbatim}
\caption{Partial GEN encoding for TC-6.}\label{fig:tc6gen}
\figrule
\end{figure}

Figures in this section include parts of all encodings mentioned in the manuscript. In all cases, only the code that encodes the Equals and Finishes relations is shown. The remaining relations are encoded in a similar manner. Complete versions of the encodings are available at https://github.com/gmparg/ICLP2018. Figures~\ref{fig:tc6coi7},~\ref{fig:tc6ctsa},~\ref{fig:tc6ctsa2} and~\ref{fig:tc6gen} contain TC-6 encodings according to COI7, CTSA, CTSA2 and GEN, respectively. Figures~\ref{fig:tc10coi7},~\ref{fig:tc10ctsa},~\ref{fig:tc10ctsa2} and~\ref{fig:tc10gen} contain TC-10 encodings according to COI7, CTSA, CTSA2 and GEN, respectively.

\clearpage

\begin{figure} [p]
\footnotesize
\figrule
\begin{verbatim}
{ s(X,Z) ; f(X,Z) ; ex(X,Z); ex(Z,X) ; alt(X,Z) ; ret(X,Z) ; rev(X,Z) ;  i(X,Z) ; eq(X,Z) ;
 dis(X,Z) }=1 :- traj(X), traj(Z), X!=Z.
eq(X,X) :- traj(X).
eq(X,Z) :- eq(X,Y), eq(Y,Z).     f(X,Z) :- f(X,Y), eq(Y,Z).
rev(X,Z) :- eq(X,Y), rev(Y,Z).   ex(Z,X) :- f(X,Y), rev(Y,Z).
alt(X,Z) :- eq(X,Y), alt(Y,Z).   f(X,Z) :- f(X,Y), alt(Y,Z).
ret(X,Z) :- eq(X,Y), ret(Y,Z).   ex(Z,Y) :- f(X,Y), ret(Y,Z).
s(X,Z) :- eq(X,Y), s(Y,Z).       ex(X,Z) | i(X,Z) | dis(X,Z) :- f(X,Y), s(Y,Z).
f(X,Z) :- eq(X,Y), f(Y,Z).       eq(X,Z) | alt(X,Z) | f(X,Z) :- f(X,Y), f(Y,Z).
ex(X,Z) :- eq(X,Y), ex(Y,Z).     s(X,Z) | i(X,Z) | dis(X,Z):- f(X,Y), ex(Y,Z).
ex(Z,X) :- eq(X,Y), ex(Z,Y).     rev(X,Z) | ret(X,Z) | ex(Z,X) :- f(X,Y), ex(Z,Y).
i(X,Z) :- eq(X,Y), i(Y,Z).       s(X,Z) | ex(X,Z) | i(X,Z) | dis(X,Z) :- f(X,Y), i(Y,Z).
dis(X,Z) :- eq(X,Y), dis(Y,Z).   s(X,Z) | ex(X,Z) | i(X,Z) | dis(X,Z) :- f(X,Y), dis(Y,Z).
:- eq(X,Z), alt(X,Z).            :- f(X,Z), alt(X,Z).
:- eq(X,Z), i(X,Z).              :- f(X,Z), i(X,Z).
:- eq(X,Z), s(X,Z).              :- f(X,Z), eq(X,Z).
:- eq(X,Z), f(X,Z).              :- f(X,Z), dis(X,Z).
:- eq(X,Z), dis(X,Z).            :- f(X,Z), ex(X,Z).
:- eq(X,Z), ex(X,Z).             :- f(X,Z), ex(Z,X).
:- eq(X,Z), ex(Z,X).             :- f(X,Z), rev(X,Z).
:- eq(X,Z), rev(X,Z).            :- f(X,Z), ret(X,Z).
:- eq(X,Z), ret(X,Z).            :- f(X,Z), s(X,Z).
\end{verbatim}
\caption{Partial COI7 encoding for TC-10.}\label{fig:tc10coi7}
\figrule
\end{figure}

\begin{figure} [p]
\footnotesize
%\figrule
\begin{verbatim}
{s(X,Y); f(X,Y); ex(X,Y); exi(X,Y); alt(X,Y); ret(X,Y); rev(X,Y); i(X,Y); eq(X,Y);
 dis(X,Y)}=1 :- traj(X), traj(Y), X<Y.
eq(X,X) :- traj(X).
:- eq(X,Y), eq(Y,Z), not eq(X,Z).   :- f(X,Y), eq(Y,Z), not f(X,Z).
:- eq(X,Y), rev(Y,Z),not rev(X,Z).  :- f(X,Y), rev(Y,Z), not exi(X,Z).
:- eq(X,Y), alt(Y,Z), not alt(X,Z). :- f(X,Y), alt(Y,Z), not f(X,Z).
:- eq(X,Y), ret(Y,Z), not ret(X,Z). :- f(X,Y), ret(Y,Z), not exi(X,Z).
:- eq(X,Y), s(Y,Z), not s(X,Z).     :- f(X,Y), s(Y,Z), not ex(X,Z), not i(X,Z), not dis(X,Z).
:- eq(X,Y), f(Y,Z), not f(X,Z).     :- f(X,Y), f(Y,Z), not eq(X,Z), not alt(X,Z), not f(X,Z).
:- eq(X,Y), ex(Y,Z), not ex(X,Z).   :- f(X,Y), ex(Y,Z), not s(X,Z), not i(X,Z), not dis(X,Z).
:- eq(X,Y), exi(Y,Z), not exi(X,Z).
:- eq(X,Y), i(Y,Z), not i(X,Z).
:- eq(X,Y), dis(Y,Z), not dis(X,Z).
:- f(X,Y), exi(Y,Z), not rev(X,Z), not ret(X,Z), not exi(X,Z).
:- f(X,Y), i(Y,Z), not s(X,Z), not ex(X,Z), not i(X,Z), not dis(X,Z).
:- f(X,Y), dis(Y,Z), not s(X,Z), not ex(X,Z), not i(X,Z), not dis(X,Z).
exi(X,Y) :- ex(Y,X), Y<X.
ex(X,Y) :- exi(Y,X), Y<X.
\end{verbatim}
\caption{Partial CTSA encoding for TC-10.}\label{fig:tc10ctsa}
\figrule
\end{figure}

\clearpage

\begin{figure}
\footnotesize
\figrule
\begin{verbatim}
{s(X,Y); f(X,Y); ex(X,Y); exi(X,Y); alt(X,Y); ret(X,Y); rev(X,Y); i(X,Y); eq(X,Y);
 dis(X,Y)}=1 :- traj(X), traj(Y), X<Y, #count{R : fact(R,X,Y)} = 0.
eq(X,X) :- traj(X).
:- eq(X,Y), eq(Y,Z), not eq(X,Z).   :- f(X,Y), s(Y,Z), not ex(X,Z), not i(X,Z), not dis(X,Z).
:- eq(X,Y), rev(Y,Z),not rev(X,Z).  :- f(X,Y), f(Y,Z), not eq(X,Z), not alt(X,Z), not f(X,Z).
:- eq(X,Y), alt(Y,Z), not alt(X,Z). :- f(X,Y), ex(Y,Z), not s(X,Z), not i(X,Z), not dis(X,Z).
:- eq(X,Y), ret(Y,Z), not ret(X,Z).
:- eq(X,Y), s(Y,Z), not s(X,Z).
:- eq(X,Y), f(Y,Z), not f(X,Z).
:- eq(X,Y), ex(Y,Z), not ex(X,Z).
:- eq(X,Y), exi(Y,Z), not exi(X,Z).
:- eq(X,Y), i(Y,Z), not i(X,Z).
:- eq(X,Y), dis(Y,Z), not dis(X,Z).
:- f(X,Y), eq(Y,Z), not f(X,Z).
:- f(X,Y), rev(Y,Z), not exi(X,Z).
:- f(X,Y), alt(Y,Z), not f(X,Z).
:- f(X,Y), ret(Y,Z), not exi(X,Z).
:- f(X,Y), exi(Y,Z), not rev(X,Z), not ret(X,Z), not exi(X,Z).
:- f(X,Y), i(Y,Z), not s(X,Z), not ex(X,Z), not i(X,Z), not dis(X,Z).
:- f(X,Y), dis(Y,Z), not s(X,Z), not ex(X,Z), not i(X,Z), not dis(X,Z).
eq(X,Y) :- fact(eq,X,Y).
f(X,Y) :- fact(f,X,Y).
\end{verbatim}
\caption{Partial CTSA2 encoding for TC-10.}\label{fig:tc10ctsa2}
\figrule
\end{figure}

\begin{figure}
\footnotesize
%\figrule
\begin{verbatim}
{true(X,R,Y) : relation(R)} = 1 :- element(X); element(Y); X != Y.
true(X,eq,X) :- element(X).
:- true(X,R1,Y); true(Y,R2,Z); not true(X,Rout,Z) : table(R1,R2,Rout).
:- possible(X,_,Y); not true(X,R,Y) : possible(X,R,Y).
relation(eq; rev; alt; ret; s; f; ex; exi; i; dis).
table(eq, eq, (eq)).                   table(f, eq, (f)).
table(eq, rev, (rev)).                 table(f, rev, (exi)).
table(eq, alt, (alt)).                 table(f, alt, (f)).
table(eq, ret, (ret)).                 table(f, ret, (exi)).
table(eq, s, (s)).                     table(f, s, (ex;i;dis)).
table(eq, f, (f)).                     table(f, f, (eq;alt;f)).
table(eq, ex, (ex)).                   table(f, ex, (s;i;dis)).
table(eq, exi, (exi)).                 table(f, exi, (rev;ret;exi)).
table(eq, i, (i)).                     table(f, i, (s;ex;i;dis)).
table(eq, dis, (dis)).                 table(f, dis, (s;ex;i;dis)).
\end{verbatim}
\caption{Partial GEN encoding for TC-10.}\label{fig:tc10gen}
\figrule
\end{figure}

\section{Additional Experiment Results}
\label{sec:additional}

In this appendix, we include and discuss additional results derived from the experiments discussed in Section~\ref{sec:eval}. Specifically, we include program size and grounding time values for all experiments, obtained by running clingo in gringo mode.

\begin{figure}
%\begin{subfigure}{.498\linewidth}
\centering
\subfloat{
\begin{tikzpicture}[node distance = 2cm, scale=0.79, transform shape]
\begin{axis}[xlabel= Trajectories,ylabel= Grounding time (s),legend pos= north west]
    \addplot+[color=color1,mark=otimes, mark size=2,error bars/.cd, y dir=both,y explicit] coordinates {
    (10,0)
    (20,0)
    (30,1.400)
    (40,4.130)
    (50,10.210)
    (60,17.400)
    (70,28.870)
    (80,42.700)
    (90,61.600)
    (100,86.730)
    (110,111.610)
    };
    \addplot+[color=color2,mark=x, mark size=2,error bars/.cd, y dir=both,y explicit] coordinates {
    (10,0)
    (20,0)
    (30,0)
    (40,0)
    (50,1.050)
    (60,2.550)
    (70,3.870)
    (80,6.210)
    (90,7.320)
    (100,9.880)
    (110,12.960)
    (120,17.080)
    (130,21.680)
    (140,27.780)
    (150,32.850)
    (160,40.970)
    (170,49.080)
    (180,60.250)
    (190,72.440)
    (200,83.580)
    (210,93.730)
    (220,108.930)
    (230,125.200)
    (240,143.530)
    (250,163.750)
    };
    \addplot+[color=color3,mark=oplus, mark size=2, error bars/.cd, y dir=both,y explicit] coordinates {
    (10,0)
    (20,0)
    (30,0.220)
    (40,0.560)
    (50,1.010)
    (60,2.070)
    (70,3.350)
    (80,4.840)
    (90,7.320)
    (100,9.370)
    (110,12.970)
    (120,17.080)
    (130,22.690)
    (140,26.760)
    (150,32.870)
    (160,40.980)
    (170,50.110)
    (180,61.271)
    (190,74.460)
    (200,82.600)
    (210,95.760)
    (220,112.020)
    (230,130.290)
    (240,149.560)
    (250,171.900)
    };
    \addplot+[color=color4,mark=+, mark size=2, error bars/.cd, y dir=both,y explicit] coordinates {
    (10,0)
    (20,0.660)
    (30,2.290)
    (40,5.280)
    (50,10.430)
    (60,18.140)
    (70,28.840)
    (80,43.060)
    (90,62.390)
    (100,85.760)
    (110,114.260)
    (120,154.910)
    (130,196.650)
    (140,247.590)
    };
    \legend{COI7, CTSA, CTSA2, GEN}
    \end{axis}
\end{tikzpicture}
\label{fig:eval-tc6-d1gtime}
}\quad
%\end{subfigure}
%\begin{subfigure}{.498\linewidth}
\subfloat{
%\centering
\begin{tikzpicture}[node distance = 2cm, scale=0.79, transform shape]
\begin{axis}[xlabel= Trajectories,ylabel= Program size (LOC),legend pos= north west]
\addplot+[color=color1,mark=otimes, mark size=2,error bars/.cd, y dir=both,y explicit] coordinates {
    (10,81)
    (20,96)
    (30,111)
    (40,126)
    (50,141)
    (60,156)
    (70,171)
    (80,186)
    (90,201)
    (100,216)
    (110,231)
    };
    \addplot+[color=color2,mark=x, mark size=2,error bars/.cd, y dir=both,y explicit] coordinates {
    (10,51)
    (20,66)
    (30,81)
    (40,96)
    (50,111)
    (60,126)
    (70,141)
    (80,156)
    (90,171)
    (100,186)
    (110,201)
    (120,216)
    (130,231)
    (140,246)
    (150,261)
    (160,276)
    (170,291)
    (180,306)
    (190,321)
    (200,336)
    (210,351)
    (220,366)
    (230,381)
    (240,396)
    (250,411)
    };
    \addplot+[color=color3,mark=oplus, mark size=2, error bars/.cd, y dir=both,y explicit] coordinates {
    (10,57)
    (20,72)
    (30,87)
    (40,102)
    (50,117)
    (60,132)
    (70,147)
    (80,162)
    (90,177)
    (100,192)
    (110,207)
    (120,222)
    (130,237)
    (140,252)
    (150,267)
    (160,282)
    (170,297)
    (180,312)
    (190,327)
    (200,342)
    (210,357)
    (220,372)
    (230,387)
    (240,402)
    (250,417)
    };
    \addplot+[color=color4,mark=+, mark size=2, error bars/.cd, y dir=both,y explicit] coordinates {
    (10,61)
    (20,76)
    (30,91)
    (40,106)
    (50,121)
    (60,136)
    (70,151)
    (80,166)
    (90,181)
    (100,196)
    (110,211)
    (120,226)
    (130,241)
    (140,256)
    };
    \legend{COI7, CTSA, CTSA2, GEN}
    \end{axis}
\end{tikzpicture}
\label{fig:eval-tc6-d1loc}
%\end{subfigure}
}
\caption{TC-6: Finding a consistent configuration when only one relation per trajectory is known.}
\label{fig:eval-tc6-d1-b}
\end{figure}
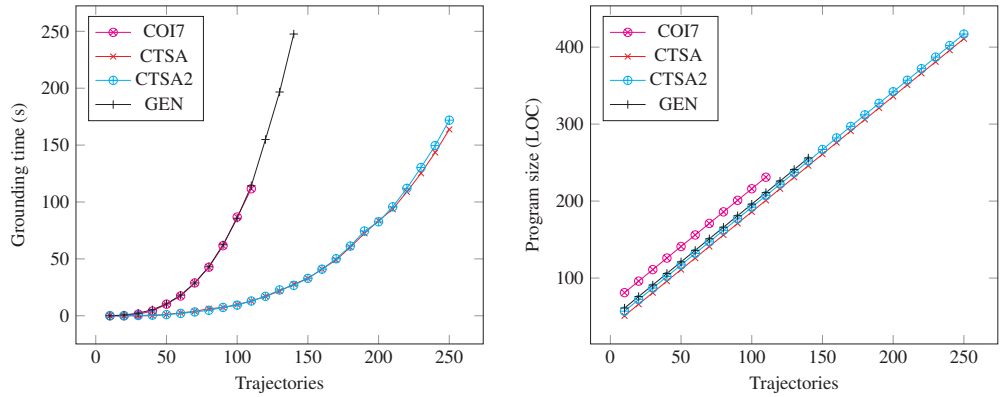

Figure~\ref{fig:eval-tc6-d1-b} shows grounding time and program size results for the first TC-6 experiment. Similarly to Fig.~\ref{fig:eval-tc6-d1}, COI7 and GEN are an order of magnitude slower than CTSA and CTSA2. GEN spends almost equal time for grounding compared to COI7, which means that the slightly longer overall CPU time for GEN shown in Fig.~\ref{fig:eval-tc6-d1time} is due to solving. The same holds for CTSA and CTSA2: the slight difference in overall CPU time is attributed more to the solving process, since the difference is less pronounced between grounding times.

In terms of program size, COI7 uses slightly more lines of code (LOC) than the other encodings due to the unnecessary integrity constraints for pairwise disjointness. The slight difference between CTSA and CTSA2 is due to rules that derive a relation whenever it is known as a fact, which are necessary due to the differentiation between known and derived facts that is introduced in CTSA2.

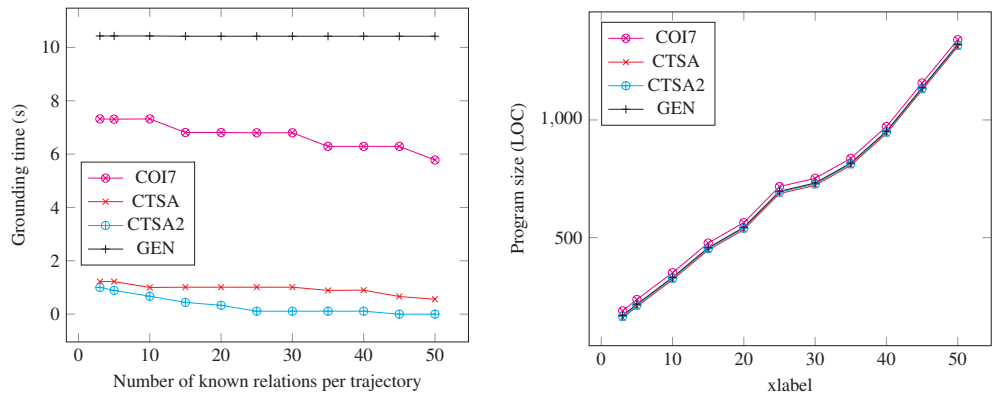
\begin{figure}
\centering
%\begin{subfigure}{.498\linewidth}
\subfloat{
\begin{tikzpicture}[node distance = 2cm, scale=0.78, transform shape]
\begin{axis}[xlabel= Number of known relations per trajectory,ylabel= Grounding time (s),legend style={at={(0.32,0.54)}}]
    \addplot+[color=color1,mark=otimes, mark size=2,error bars/.cd, y dir=both,y explicit] coordinates {
    (3,7.320)
    (5,7.310)
    (10,7.320)
    (15,6.810)
    (20,6.810)
    (25,6.800)
    (30,6.800)
    (35,6.290)
    (40,6.290)
    (45,6.290)
    (50,5.780)
    };
    \addplot+[color=color2,mark=x, mark size=2,error bars/.cd, y dir=both,y explicit] coordinates {
    (3,1.220)
    (5,1.220)
    (10,1.000)
    (15,1.010)
    (20,1.010)
    (25,1.010)
    (30,1.010)
    (35,0.890)
    (40,0.900)
    (45,0.660)
    (50,0.560)
    };
    \addplot+[color=color3,mark=oplus, mark size=2, error bars/.cd, y dir=both,y explicit] coordinates {
    (3,1.000)
    (5,0.890)
    (10,0.670)
    (15,0.440)
    (20,0.330)
    (25,0.110)
    (30,0.110)
    (35,0.110)
    (40,0.110)
    (45,0)
    (50,0)
    };
    \addplot+[color=color4,mark=+, mark size=2, error bars/.cd, y dir=both,y explicit] coordinates {
    (3,10.430)
    (5,10.430)
    (10,10.430)
    (15,10.420)
    (20,10.420)
    (25,10.420)
    (30,10.420)
    (35,10.420)
    (40,10.420)
    (45,10.420)
    (50,10.420)
    };
    \legend{COI7, CTSA, CTSA2, GEN}
    \end{axis}
\end{tikzpicture}
\label{fig:eval-tc6-50gtime}
%\end{subfigure}
}\quad
\subfloat{
%\begin{subfigure}{.498\linewidth}
%\centering
\begin{tikzpicture}[node distance = 2cm, scale=0.78, transform shape]
\begin{axis}[xlabel= xlabel= Number of known relations per trajectory,ylabel= Program size (LOC),legend pos= north west]
    \addplot+[color=color1,mark=otimes, mark size=2,error bars/.cd, y dir=both,y explicit] coordinates {
    (3,189)
    (5,237)
    (10,351)
    (15,477)
    (20,564)
    (25,717)
    (30,752)
    (35,837)
    (40,972)
    (45,1157)
    (50,1341)
    };
    \addplot+[color=color2,mark=x, mark size=2,error bars/.cd, y dir=both,y explicit] coordinates {
    (3,159)
    (5,207)
    (10,321)
    (15,447)
    (20,534)
    (25,687)
    (30,722)
    (35,807)
    (40,942)
    (45,1127)
    (50,1311)
    };
    \addplot+[color=color3,mark=oplus, mark size=2, error bars/.cd, y dir=both,y explicit] coordinates {
    (3,165)
    (5,213)
    (10,327)
    (15,453)
    (20,540)
    (25,693)
    (30,728)
    (35,813)
    (40,948)
    (45,1133)
    (50,1317)
    };
    \addplot+[color=color4,mark=+, mark size=2, error bars/.cd, y dir=both,y explicit] coordinates {
    (3,169)
    (5,217)
    (10,331)
    (15,457)
    (20,544)
    (25,697)
    (30,732)
    (35,817)
    (40,952)
    (45,1137)
    (50,1321)
    };
    \legend{COI7, CTSA, CTSA2, GEN}
    \end{axis}
\end{tikzpicture}
\label{fig:eval-tc6-50loc}
%\end{subfigure}
}
\caption{TC-6: Determining consistency for 50 trajectories.}
\label{fig:eval-tc6-50-b}
\end{figure}

Grounding time and program size results for the second TC-6 experiment are shown in Fig.~\ref{fig:eval-tc6-50-b}. It is evident that GEN has a constant grounding time regardless of the number of known relations. This is because each known relation adds a possible fact, which only features in a single integrity constraint with variables. In the case of COI7, on the other hand, each known relation adds a specific relation fact, which takes part in multiple rules that form the composition table entries.

In what concerns program size, differences between encodings appear less significant, since the number of LOC increases as more relations per trajectory are known. Note that the plots are not straight lines due to the fact that knowing x out of 50 possible relations per trajectory is not exactly equal to knowing 2x\% of all possible relations because each relation is between two trajectories. For instance, knowing 25 relations per trajectory can be achieved by knowing the relations for 601 unique trajectory pairs, which is slightly less than 50\% of all possible pairs ($\binom{50}{2}=\frac{50!}{2!(50-2)!}=1225$ possible pairs).

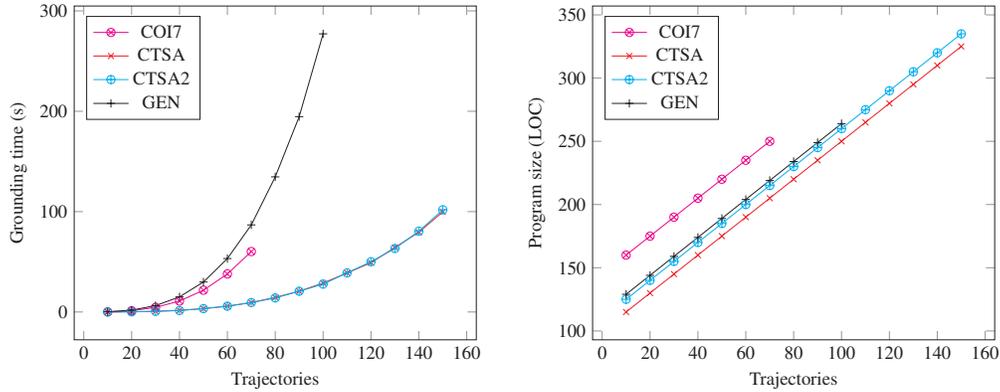
\begin{figure}
\centering
%\begin{subfigure}{.498\linewidth}
\subfloat{
\begin{tikzpicture}[node distance = 2cm, scale=0.78, transform shape]
\begin{axis}[xlabel= Trajectories,ylabel= Grounding time (s),legend pos= north west]
    \addplot+[color=color1,mark=otimes, mark size=2,error bars/.cd, y dir=both,y explicit] coordinates {
    (10,0.110)
    (20,1.220)
    (30,4.630)
    (40,10.910)
    (50,21.680)
    (60,37.910)
    (70,60.240)
    };
    \addplot+[color=color2,mark=x, mark size=2,error bars/.cd, y dir=both,y explicit] coordinates {
    (10,0)
    (20,0.220)
    (30,0.670)
    (40,1.650)
    (50,3.350)
    (60,5.780)
    (70,9.380)
    (80,14.510)
    (90,20.670)
    (100,28.780)
    (110,38.930)
    (120,49.080)
    (130,64.280)
    (140,79.520)
    (150,99.810)
    };
    \addplot+[color=color3,mark=oplus, mark size=2, error bars/.cd, y dir=both,y explicit] coordinates {
    (10,0)
    (20,0.110)
    (30,0.670)
    (40,1.430)
    (50,3.350)
    (60,5.770)
    (70,9.380)
    (80,13.990)
    (90,20.680)
    (100,27.770)
    (110,38.950)
    (120,50.090)
    (130,63.280)
    (140,80.530)
    (150,101.850)
    };
     \addplot+[color=color4,mark=+, mark size=2, error bars/.cd, y dir=both,y explicit] coordinates {
    (10,0.220)
    (20,1.870)
    (30,6.320)
    (40,15.070)
    (50,29.870)
    (60,53.240)
    (70,86.790)
    (80,134.630)
    (90,194.650)
    (100,277.160)
    };
    \legend{COI7, CTSA, CTSA2, GEN}
    \end{axis}
\end{tikzpicture}
\label{fig:eval-tc10-d1gtime}
%\end{subfigure}
}\quad
\subfloat{
%\begin{subfigure}{.498\linewidth}
%\centering
\begin{tikzpicture}[node distance = 2cm, scale=0.78, transform shape]
\begin{axis}[xlabel= Trajectories,ylabel= Program size (LOC),legend pos= north west]
    \addplot+[color=color1,mark=otimes, mark size=2,error bars/.cd, y dir=both,y explicit] coordinates {
    (10,160)
    (20,175)
    (30,190)
    (40,205)
    (50,220)
    (60,235)
    (70,250)
    };
    \addplot+[color=color2,mark=x, mark size=2,error bars/.cd, y dir=both,y explicit] coordinates {
    (10,115)
    (20,130)
    (30,145)
    (40,160)
    (50,175)
    (60,190)
    (70,205)
    (80,220)
    (90,235)
    (100,250)
    (110,265)
    (120,280)
    (130,295)
    (140,310)
    (150,325)
    };
    \addplot+[color=color3,mark=oplus, mark size=2, error bars/.cd, y dir=both,y explicit] coordinates {
    (10,125)
    (20,140)
    (30,155)
    (40,170)
    (50,185)
    (60,200)
    (70,215)
    (80,230)
    (90,245)
    (100,260)
    (110,275)
    (120,290)
    (130,305)
    (140,320)
    (150,335)
    };
    \addplot+[color=color4,mark=+, mark size=2, error bars/.cd, y dir=both,y explicit] coordinates {
    (10,129)
    (20,144)
    (30,159)
    (40,174)
    (50,189)
    (60,204)
    (70,219)
    (80,234)
    (90,249)
    (100,264)
    };
    \legend{COI7, CTSA, CTSA2, GEN}
    \end{axis}
\end{tikzpicture}
\label{fig:eval-tc10-d1loc}
%\end{subfigure}
}
\caption{TC-10: Finding a consistent configuration when only one relation per trajectory is known.}
\label{fig:eval-tc10-d1-b}
\end{figure}

Figure~\ref{fig:eval-tc6-d1-b} shows grounding time and program size results for the first TC-10 experiment. Again, grounding times for COI7 and GEN are an order of magnitude longer than those of CTSA and CTSA2, which exhibit almost identical times (as is the case for overall CPU time in Fig.~\ref{fig:eval-tc10-d1}), as their differences only affect grounding times when more relations are known per trajectory. Interestingly, GEN's grounding time is slightly higher than COI7, which means that solving time must be slightly lower, since their overall CPU times do not differ much.

The increased number of relations in TC-10 leads to more discernible differences in program size among encodings, which are again due to the same reasons as in TC-6: COI7 has highest LOC counts due to the extra integrity constraints for pairwise disjointness and CTSA2 needs more LOC than CTSA to derive a relation whenever it is known as a fact, for each base relation.

\begin{figure}
\centering
%\begin{subfigure}{.498\linewidth}
\subfloat{
\begin{tikzpicture}[node distance = 2cm, scale=0.78, transform shape]
\begin{axis}[xlabel= Number known relations per trajectory,ylabel= Grounding time (s),legend style={at={(0.32,0.54)}}]
    \addplot+[color=color1,mark=otimes, mark size=2,error bars/.cd, y dir=both,y explicit] coordinates {
    (3,21.670)
    (5,20.660)
    (10,20.650)
    (15,20.650)
    (20,20.140)
    (25,20.140)
    (30,19.630)
    (35,19.620)
    (40,19.120)
    (45,18.080)
    (50,17.590)
    };
    \addplot+[color=color2,mark=x, mark size=2,error bars/.cd, y dir=both,y explicit] coordinates {
    (3,3.360)
    (5,3.340)
    (10,3.340)
    (15,3.340)
    (20,3.120)
    (25,3.130)
    (30,3.130)
    (35,2.920)
    (40,2.710)
    (45,2.280)
    (50,2.070)
    };
    \addplot+[color=color3,mark=oplus, mark size=2, error bars/.cd, y dir=both,y explicit] coordinates {
    (3,2.920)
    (5,2.490)
    (10,1.860)
    (15,1.220)
    (20,0.900)
    (25,0.330)
    (30,0.330)
    (35,0.220)
    (40,0.220)
    (45,0.110)
    (50,0)
    };
    \addplot+[color=color4,mark=+, mark size=2, error bars/.cd, y dir=both,y explicit] coordinates {
    (3,29.860)
    (5,29.850)
    (10,29.870)
    (15,29.860)
    (20,29.860)
    (25,29.870)
    (30,29.860)
    (35,29.870)
    (40,29.850)
    (45,29.860)
    (50,29.850)
    };
    \legend{COI7, CTSA, CTSA2, GEN}
    \end{axis}
\end{tikzpicture}
\label{fig:eval-tc10-50gtime}
%\end{subfigure}
}\quad
\subfloat{
%\begin{subfigure}{.498\linewidth}
%\centering
\begin{tikzpicture}[node distance = 2cm, scale=0.78, transform shape]
\begin{axis}[xlabel= Number of known relations per trajectory,ylabel= Program size (LOC),legend pos= north west]
    \addplot+[color=color1,mark=otimes, mark size=2,error bars/.cd, y dir=both,y explicit] coordinates {
    (3,268)
    (5,316)
    (10,430)
    (15,556)
    (20,643)
    (25,796)
    (30,831)
    (35,916)
    (40,1051)
    (45,1236)
    (50,1420)
    };
    \addplot+[color=color2,mark=x, mark size=2,error bars/.cd, y dir=both,y explicit] coordinates {
    (3,223)
    (5,271)
    (10,385)
    (15,511)
    (20,598)
    (25,751)
    (30,786)
    (35,871)
    (40,1006)
    (45,1191)
    (50,1375)
    };
    \addplot+[color=color3,mark=oplus, mark size=2, error bars/.cd, y dir=both,y explicit] coordinates {
    (3,233)
    (5,281)
    (10,395)
    (15,521)
    (20,608)
    (25,761)
    (30,796)
    (35,881)
    (40,1016)
    (45,1201)
    (50,1385)
    };
\addplot+[color=color4,mark=+, mark size=2, error bars/.cd, y dir=both,y explicit] coordinates {
    (3,237)
    (5,265)
    (10,399)
    (15,525)
    (20,612)
    (25,765)
    (30,800)
    (35,885)
    (40,1020)
    (45,1205)
    (50,1389)
    };
    \legend{COI7, CTSA, CTSA2, GEN}
    \end{axis}
\end{tikzpicture}
\label{fig:eval-tc10-50loc}
%\end{subfigure}
}
\caption{TC-10: Determining consistency for 50 trajectories.}
\label{fig:eval-tc10-50-b}
\end{figure}
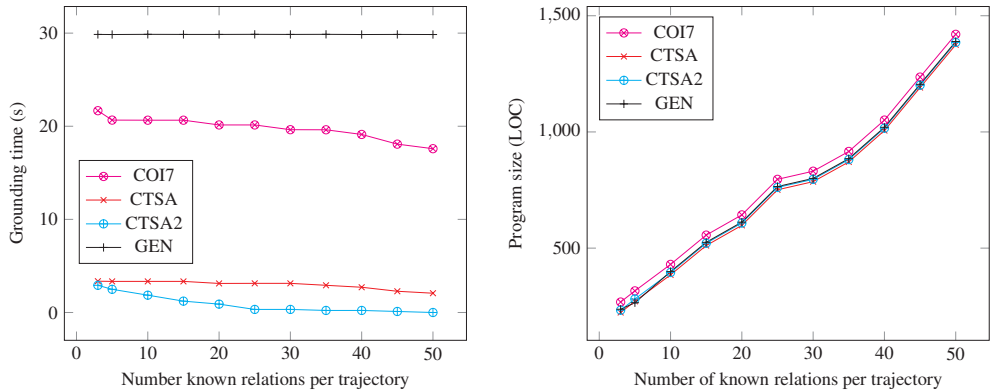

Grounding time and program size results for the second TC-10 experiment are shown in Fig.~\ref{fig:eval-tc10-50-b}. The results mirror those of TC-6, which is expected since an increase from 6 to 10 base relations does not affect the way the encodings behave as more relations are known per trajectory.

\label{lastpage}

\begin{thebibliography}{31}

%%% the Web address.

\ifx \showCODEN    \undefined \def \showCODEN     #1{\unskip}     \fi
\ifx \showDOI      \undefined \def \showDOI       #1{#1}\fi
\ifx \showISBNx    \undefined \def \showISBNx     #1{\unskip}     \fi
\ifx \showISBNxiii \undefined \def \showISBNxiii  #1{\unskip}     \fi
\ifx \showISSN     \undefined \def \showISSN      #1{\unskip}     \fi
\ifx \showLCCN     \undefined \def \showLCCN      #1{\unskip}     \fi
\ifx \shownote     \undefined \def \shownote      #1{#1}          \fi
\ifx \showarticletitle \undefined \def \showarticletitle #1{#1}   \fi
\ifx \showURL      \undefined \def \showURL       {\relax}        \fi
% invisible to TeX
\providecommand\bibfield[2]{#2}
\providecommand\bibinfo[2]{#2}
\providecommand\natexlab[1]{#1}
\providecommand\showeprint[2][]{arXiv:#2}

\bibitem[\protect\citeauthoryear{Allen}{Allen}{1981}]%
        {Allen1981}
\bibfield{author}{\bibinfo{person}{James~F. Allen}.}
  \bibinfo{year}{1981}\natexlab{}.
\newblock \showarticletitle{An Interval-Based Representation of Temporal
  Knowledge.}. In \bibinfo{booktitle}{\emph{IJCAI}},
  \bibfield{editor}{\bibinfo{person}{Patrick~J. Hayes}} (Ed.).
  \bibinfo{publisher}{William Kaufmann}, \bibinfo{pages}{221--226}.
\newblock


\bibitem[\protect\citeauthoryear{Andrienko, Andrienko, Fuchs, and
  Wood}{Andrienko et~al\mbox{.}}{2017}]%
        {Andrienko2017}
\bibfield{author}{\bibinfo{person}{Gennady~L. Andrienko},
  \bibinfo{person}{Natalia~V. Andrienko}, \bibinfo{person}{Georg Fuchs}, {and}
  \bibinfo{person}{Jo Wood}.} \bibinfo{year}{2017}\natexlab{}.
\newblock \showarticletitle{Revealing Patterns and Trends of Mass Mobility
  Through Spatial and Temporal Abstraction of Origin-Destination Movement
  Data.}
\newblock \bibinfo{journal}{\emph{IEEE Trans. Vis. Comput. Graph.}}
  \bibinfo{volume}{23}, \bibinfo{number}{9} (\bibinfo{year}{2017}),
  \bibinfo{pages}{2120--2136}.
\newblock


\bibitem[\protect\citeauthoryear{Brenton, Faber, and Batsakis}{Brenton
  et~al\mbox{.}}{2016}]%
        {Brenton2016}
\bibfield{author}{\bibinfo{person}{Christopher Brenton},
  \bibinfo{person}{Wolfgang Faber}, {and} \bibinfo{person}{Sotiris Batsakis}.}
  \bibinfo{year}{2016}\natexlab{}.
\newblock \showarticletitle{Answer Set Programming for Qualitative
  Spatio-Temporal Reasoning: Methods and Experiments.}. In
  \bibinfo{booktitle}{\emph{ICLP (Technical Communications)}}
  \emph{(\bibinfo{series}{OASICS})}, \bibfield{editor}{\bibinfo{person}{Manuel
  Carro}, \bibinfo{person}{Andy King}, \bibinfo{person}{Neda Saeedloei}, {and}
  \bibinfo{person}{Marina~De Vos}} (Eds.), Vol.~\bibinfo{volume}{52}.
  \bibinfo{publisher}{Schloss Dagstuhl - Leibniz-Zentrum fuer Informatik},
  \bibinfo{pages}{4:1--4:15}.
\newblock
\showISBNx{978-3-95977-007-1}


\bibitem[\protect\citeauthoryear{Cadoli and Schaerf}{Cadoli and
  Schaerf}{1993}]%
        {DBLP:journals/jlp/CadoliS93}
\bibfield{author}{\bibinfo{person}{Marco Cadoli} {and} \bibinfo{person}{Marco
  Schaerf}.} \bibinfo{year}{1993}\natexlab{}.
\newblock \showarticletitle{A Survey of Complexity Results for Nonmonotonic
  Logics}.
\newblock \bibinfo{journal}{\emph{J. Log. Program.}} \bibinfo{volume}{17},
  \bibinfo{number}{2/3{\&}4} (\bibinfo{year}{1993}), \bibinfo{pages}{127--160}.
\newblock


\bibitem[\protect\citeauthoryear{Cohn, Bennett, Gooday, and Gotts}{Cohn
  et~al\mbox{.}}{1997}]%
        {Cohn1997}
\bibfield{author}{\bibinfo{person}{Anthony~G. Cohn}, \bibinfo{person}{Brandon
  Bennett}, \bibinfo{person}{John Gooday}, {and} \bibinfo{person}{Nicholas~Mark
  Gotts}.} \bibinfo{year}{1997}\natexlab{}.
\newblock \showarticletitle{Qualitative Spatial Representation and Reasoning
  with the Region Connection Calculus}.
\newblock \bibinfo{journal}{\emph{GeoInformatica}} \bibinfo{volume}{1},
  \bibinfo{number}{3} (\bibinfo{year}{1997}), \bibinfo{pages}{275--316}.
\newblock


\bibitem[\protect\citeauthoryear{Delafontaine, Bogaert, Cohn, Witlox, Maeyer,
  and de~Weghe}{Delafontaine et~al\mbox{.}}{2011a}]%
        {Delafontaine2011}
\bibfield{author}{\bibinfo{person}{Matthias Delafontaine},
  \bibinfo{person}{Peter Bogaert}, \bibinfo{person}{Anthony~G. Cohn},
  \bibinfo{person}{Frank Witlox}, \bibinfo{person}{Philippe~De Maeyer}, {and}
  \bibinfo{person}{Nico~Van de Weghe}.} \bibinfo{year}{2011}\natexlab{a}.
\newblock \showarticletitle{Inferring additional knowledge from QTCN
  relations.}
\newblock \bibinfo{journal}{\emph{Inf. Sci.}} \bibinfo{volume}{181},
  \bibinfo{number}{9} (\bibinfo{year}{2011}), \bibinfo{pages}{1573--1590}.
\newblock


\bibitem[\protect\citeauthoryear{Delafontaine, Cohn, and de~Weghe}{Delafontaine
  et~al\mbox{.}}{2011b}]%
        {Delafontaine2011B}
\bibfield{author}{\bibinfo{person}{Matthias Delafontaine},
  \bibinfo{person}{Anthony~G. Cohn}, {and} \bibinfo{person}{Nico~Van de
  Weghe}.} \bibinfo{year}{2011}\natexlab{b}.
\newblock \showarticletitle{Implementing a qualitative calculus to analyse
  moving point objects.}
\newblock \bibinfo{journal}{\emph{Expert Syst. Appl.}} \bibinfo{volume}{38},
  \bibinfo{number}{5} (\bibinfo{year}{2011}), \bibinfo{pages}{5187--5196}.
\newblock


\bibitem[\protect\citeauthoryear{Eiter, Gottlob, and Mannila}{Eiter
  et~al\mbox{.}}{1997}]%
        {Eiter1997}
\bibfield{author}{\bibinfo{person}{Thomas Eiter}, \bibinfo{person}{Georg
  Gottlob}, {and} \bibinfo{person}{Heikki Mannila}.}
  \bibinfo{year}{1997}\natexlab{}.
\newblock \showarticletitle{Disjunctive Datalog}.
\newblock \bibinfo{journal}{\emph{{ACM} Trans. Database Syst.}}
  \bibinfo{volume}{22}, \bibinfo{number}{3} (\bibinfo{year}{1997}),
  \bibinfo{pages}{364--418}.
\newblock


\bibitem[\protect\citeauthoryear{Escrig and Toledo}{Escrig and Toledo}{2002}]%
        {Escrig2002}
\bibfield{author}{\bibinfo{person}{M.~Teresa Escrig} {and}
  \bibinfo{person}{Francisco Toledo}.} \bibinfo{year}{2002}\natexlab{}.
\newblock \showarticletitle{Qualitative Velocity.}. In
  \bibinfo{booktitle}{\emph{CCIA}} \emph{(\bibinfo{series}{Lecture Notes in
  Computer Science})}, \bibfield{editor}{\bibinfo{person}{M.~Teresa Escrig},
  \bibinfo{person}{Francisco Toledo}, {and} \bibinfo{person}{Elisabet
  Golobardes}} (Eds.), Vol.~\bibinfo{volume}{2504}.
  \bibinfo{publisher}{Springer}, \bibinfo{pages}{29--39}.
\newblock
\showISBNx{3-540-00011-9}


\bibitem[\protect\citeauthoryear{Freuder and Mackworth}{Freuder and
  Mackworth}{2006}]%
        {Freuder2006}
\bibfield{author}{\bibinfo{person}{Eugene~C. Freuder} {and}
  \bibinfo{person}{Alan~K. Mackworth}.} \bibinfo{year}{2006}\natexlab{}.
\newblock \showarticletitle{{Constraint Satisfaction: An Emerging Paradigm}}.
\newblock In \bibinfo{booktitle}{\emph{Handbook of Constraint Programming}},
  \bibfield{editor}{\bibinfo{person}{Francesca Rossi}, \bibinfo{person}{Peter
  van Beek}, {and} \bibinfo{person}{Toby Walsh}} (Eds.).
  \bibinfo{series}{Foundations of Artificial Intelligence},
  Vol.~\bibinfo{volume}{2}. \bibinfo{publisher}{Elsevier Science Inc.},
  \bibinfo{address}{New York, NY, USA}, \bibinfo{pages}{13--27}.
\newblock
\showISSN{1574-6526}


\bibitem[\protect\citeauthoryear{Gebser, Harrison, Kaminski, Lifschitz, and
  Schaub}{Gebser et~al\mbox{.}}{2015}]%
        {Gebser2015}
\bibfield{author}{\bibinfo{person}{Martin Gebser}, \bibinfo{person}{Amelia
  Harrison}, \bibinfo{person}{Roland Kaminski}, \bibinfo{person}{Vladimir
  Lifschitz}, {and} \bibinfo{person}{Torsten Schaub}.}
  \bibinfo{year}{2015}\natexlab{}.
\newblock \showarticletitle{Abstract gringo}.
\newblock \bibinfo{journal}{\emph{{TPLP}}} \bibinfo{volume}{15},
  \bibinfo{number}{4-5} (\bibinfo{year}{2015}), \bibinfo{pages}{449--463}.
\newblock


\bibitem[\protect\citeauthoryear{Gebser, Kaminski, Kaufmann, Ostrowski, Schaub,
  and Wanko}{Gebser et~al\mbox{.}}{2016}]%
        {Gebser2016}
\bibfield{author}{\bibinfo{person}{Martin Gebser}, \bibinfo{person}{Roland
  Kaminski}, \bibinfo{person}{Benjamin Kaufmann}, \bibinfo{person}{Max
  Ostrowski}, \bibinfo{person}{Torsten Schaub}, {and} \bibinfo{person}{Philipp
  Wanko}.} \bibinfo{year}{2016}\natexlab{}.
\newblock \showarticletitle{Theory Solving Made Easy with Clingo~5}. In
  \bibinfo{booktitle}{\emph{Technical Communications of the Thirty-second
  International Conference on Logic Programming (ICLP'16)}}
  \emph{(\bibinfo{series}{Open Access Series in Informatics (OASIcs)})},
  \bibfield{editor}{\bibinfo{person}{Manuel Carro} {and} \bibinfo{person}{Andy
  King}} (Eds.), Vol.~\bibinfo{volume}{52}. \bibinfo{publisher}{Schloss
  Dagstuhl}, \bibinfo{pages}{2:1--2:15}.
\newblock


\bibitem[\protect\citeauthoryear{Hazarika and Cohn}{Hazarika and Cohn}{2001}]%
        {Hazarika2001}
\bibfield{author}{\bibinfo{person}{Shyamanta~M. Hazarika} {and}
  \bibinfo{person}{Anthony~G. Cohn}.} \bibinfo{year}{2001}\natexlab{}.
\newblock \showarticletitle{Qualitative Spatio-Temporal Continuity.}. In
  \bibinfo{booktitle}{\emph{COSIT}} \emph{(\bibinfo{series}{Lecture Notes in
  Computer Science})}, \bibfield{editor}{\bibinfo{person}{Daniel~R. Montello}}
  (Ed.), Vol.~\bibinfo{volume}{2205}. \bibinfo{publisher}{Springer},
  \bibinfo{pages}{92--107}.
\newblock
\showISBNx{3-540-42613-2}


\bibitem[\protect\citeauthoryear{Hu, Peeta, and Liou}{Hu et~al\mbox{.}}{2016}]%
        {Hu2016}
\bibfield{author}{\bibinfo{person}{Shou-Ren Hu}, \bibinfo{person}{Srinivas
  Peeta}, {and} \bibinfo{person}{Han-Tsung Liou}.}
  \bibinfo{year}{2016}\natexlab{}.
\newblock \showarticletitle{Integrated Determination of Network
  Origin-Destination Trip Matrix and Heterogeneous Sensor Selection and
  Location Strategy.}
\newblock \bibinfo{journal}{\emph{IEEE Trans. Intelligent Transportation
  Systems}} \bibinfo{volume}{17}, \bibinfo{number}{1} (\bibinfo{year}{2016}),
  \bibinfo{pages}{195--205}.
\newblock


\bibitem[\protect\citeauthoryear{Li}{Li}{2012}]%
        {Li2012}
\bibfield{author}{\bibinfo{person}{Jason~Jingshi Li}.}
  \bibinfo{year}{2012}\natexlab{}.
\newblock \showarticletitle{Qualitative Spatial and Temporal Reasoning with
  Answer Set Programming.}. In \bibinfo{booktitle}{\emph{ICTAI}}.
  \bibinfo{publisher}{IEEE Computer Society}, \bibinfo{pages}{603--609}.
\newblock
\showISBNx{978-1-4799-0227-9}


\bibitem[\protect\citeauthoryear{Ma, Lu, Liu, Wang, and Xiong}{Ma
  et~al\mbox{.}}{2017}]%
        {Ma2017}
\bibfield{author}{\bibinfo{person}{Zijian Ma}, \bibinfo{person}{Dan Lu},
  \bibinfo{person}{Qian Liu}, \bibinfo{person}{Jingyuan Wang}, {and}
  \bibinfo{person}{Zhang Xiong}.} \bibinfo{year}{2017}\natexlab{}.
\newblock \showarticletitle{City-Eyes: A multi-source data integration basec
  smart city analysis system.}. In \bibinfo{booktitle}{\emph{Proceedings of the
  IEEE 18th International Symposium on A World of Wireless, Mobile and
  Multimedia Networks (WoWMoM2017)}}. \bibinfo{publisher}{IEEE},
  \bibinfo{pages}{1--3}.
\newblock
\showISBNx{978-1-5386-2723-5}


\bibitem[\protect\citeauthoryear{Marek and Truszczynski}{Marek and
  Truszczynski}{1989}]%
        {DBLP:conf/slp/MarekT89}
\bibfield{author}{\bibinfo{person}{V.~Wiktor Marek} {and}
  \bibinfo{person}{Miroslaw Truszczynski}.} \bibinfo{year}{1989}\natexlab{}.
\newblock \showarticletitle{Stable Semantics for Logic Programs and Default
  Theories}. In \bibinfo{booktitle}{\emph{Logic Programming, Proceedings of the
  North American Conference 1989, Cleveland, Ohio, USA, October 16-20, 1989. 2
  Volumes}}, \bibfield{editor}{\bibinfo{person}{Ewing~L. Lusk} {and}
  \bibinfo{person}{Ross~A. Overbeek}} (Eds.). \bibinfo{publisher}{{MIT} Press},
  \bibinfo{pages}{243--256}.
\newblock


\bibitem[\protect\citeauthoryear{Mart\'{i}nez-Mart\'{i}n, Escrig, and del
  Pobil}{Mart\'{i}nez-Mart\'{i}n et~al\mbox{.}}{2012}]%
        {Martinez2012}
\bibfield{author}{\bibinfo{person}{Ester Mart\'{i}nez-Mart\'{i}n},
  \bibinfo{person}{M.~Teresa Escrig}, {and} \bibinfo{person}{Angel~Pasqual del
  Pobil}.} \bibinfo{year}{2012}\natexlab{}.
\newblock \showarticletitle{A General Qualitative Spatio-Temporal Model Based
  on Intervals.}
\newblock \bibinfo{journal}{\emph{J. UCS}} \bibinfo{volume}{18},
  \bibinfo{number}{10} (\bibinfo{year}{2012}), \bibinfo{pages}{1343--1378}.
\newblock


\bibitem[\protect\citeauthoryear{Moreira-Matias, Gama, Ferreira,
  Mendes-Moreira, and Damas}{Moreira-Matias et~al\mbox{.}}{2013}]%
        {Moreira2013}
\bibfield{author}{\bibinfo{person}{Luís Moreira-Matias},
  \bibinfo{person}{João Gama}, \bibinfo{person}{Michel Ferreira},
  \bibinfo{person}{João Mendes-Moreira}, {and} \bibinfo{person}{Luís Damas}.}
  \bibinfo{year}{2013}\natexlab{}.
\newblock \showarticletitle{Predicting Taxi-Passenger Demand Using Streaming
  Data.}
\newblock \bibinfo{journal}{\emph{IEEE Trans. Intelligent Transportation
  Systems}} \bibinfo{volume}{14}, \bibinfo{number}{3} (\bibinfo{year}{2013}),
  \bibinfo{pages}{1393--1402}.
\newblock


\bibitem[\protect\citeauthoryear{Muller}{Muller}{1998}]%
        {Muller1998}
\bibfield{author}{\bibinfo{person}{Philippe Muller}.}
  \bibinfo{year}{1998}\natexlab{}.
\newblock \showarticletitle{A Qualitative Theory of Motion Based on
  Spatio-Temporal Primitives.}. In \bibinfo{booktitle}{\emph{KR}},
  \bibfield{editor}{\bibinfo{person}{Anthony~G. Cohn},
  \bibinfo{person}{Lenhart~K. Schubert}, {and} \bibinfo{person}{Stuart~C.
  Shapiro}} (Eds.). \bibinfo{publisher}{Morgan Kaufmann},
  \bibinfo{pages}{131--143}.
\newblock


\bibitem[\protect\citeauthoryear{Noyon, Claramunt, and Devogele}{Noyon
  et~al\mbox{.}}{2007}]%
        {Noyon2007}
\bibfield{author}{\bibinfo{person}{Valérie Noyon}, \bibinfo{person}{Christophe
  Claramunt}, {and} \bibinfo{person}{Thomas Devogele}.}
  \bibinfo{year}{2007}\natexlab{}.
\newblock \showarticletitle{A Relative Representation of Trajectories in
  Geogaphical Spaces.}
\newblock \bibinfo{journal}{\emph{GeoInformatica}} \bibinfo{volume}{11},
  \bibinfo{number}{4} (\bibinfo{year}{2007}), \bibinfo{pages}{479--496}.
\newblock


\bibitem[\protect\citeauthoryear{Patroumpas and Sellis}{Patroumpas and
  Sellis}{2015}]%
        {Patroumpas2015}
\bibfield{author}{\bibinfo{person}{Kostas Patroumpas} {and}
  \bibinfo{person}{Timos~K. Sellis}.} \bibinfo{year}{2015}\natexlab{}.
\newblock \showarticletitle{Managing Big Trajectory Data: Online Processing of
  Positional Streams.}
\newblock In \bibinfo{booktitle}{\emph{Big Data - Algorithms, Analytics, and
  Applications}}, \bibfield{editor}{\bibinfo{person}{Kuan-Ching Li},
  \bibinfo{person}{Hai Jiang}, \bibinfo{person}{Laurence~T. Yang}, {and}
  \bibinfo{person}{Alfredo Cuzzocrea}} (Eds.). \bibinfo{publisher}{Chapman and
  Hall/CRC}, \bibinfo{pages}{257--280}.
\newblock
\showISBNx{978-1-4822-4056-6}


\bibitem[\protect\citeauthoryear{Renz}{Renz}{2002}]%
        {Renz2002}
\bibfield{editor}{\bibinfo{person}{Jochen Renz}} (Ed.).
  \bibinfo{year}{2002}\natexlab{}.
\newblock \bibinfo{booktitle}{\emph{The Region Connection Calculus}}.
\newblock \bibinfo{publisher}{Springer Berlin Heidelberg},
  \bibinfo{pages}{41--50}.
\newblock
\showISBNx{978-3-540-70736-3}


\bibitem[\protect\citeauthoryear{Renz, Rauh, and Knauff}{Renz
  et~al\mbox{.}}{2000}]%
        {Renz2000}
\bibfield{author}{\bibinfo{person}{Jochen Renz}, \bibinfo{person}{Reinhold
  Rauh}, {and} \bibinfo{person}{Markus Knauff}.}
  \bibinfo{year}{2000}\natexlab{}.
\newblock \showarticletitle{Towards Cognitive Adequacy of Topological Spatial
  Relations.}. In \bibinfo{booktitle}{\emph{Spatial Cognition}}
  \emph{(\bibinfo{series}{Lecture Notes in Computer Science})},
  \bibfield{editor}{\bibinfo{person}{Christian Freksa},
  \bibinfo{person}{Wilfried Brauer}, \bibinfo{person}{Christopher Habel}, {and}
  \bibinfo{person}{Karl~Friedrich Wender}} (Eds.), Vol.~\bibinfo{volume}{1849}.
  \bibinfo{publisher}{Springer}, \bibinfo{pages}{184--197}.
\newblock
\showISBNx{3-540-67584-1}


\bibitem[\protect\citeauthoryear{Van~de Weghe}{Van~de Weghe}{2017}]%
        {Weghe2017}
\bibfield{author}{\bibinfo{person}{Nico Van~de Weghe}.}
  \bibinfo{year}{2017}\natexlab{}.
\newblock \showarticletitle{Spatiotemporal Relations for Moving Objects.}
\newblock In \bibinfo{booktitle}{\emph{Encyclopedia of GIS}},
  \bibfield{editor}{\bibinfo{person}{Shashi Shekhar}, \bibinfo{person}{Hui
  Xiong}, {and} \bibinfo{person}{Xun Zhou}} (Eds.).
  \bibinfo{publisher}{Springer}, \bibinfo{pages}{2177--2186}.
\newblock
\showISBNx{978-3-319-17885-1}


\bibitem[\protect\citeauthoryear{Van~de Weghe, Cohn, De~Tr\'{e}, and
  De~Maeyer}{Van~de Weghe et~al\mbox{.}}{2006}]%
        {Weghe2006}
\bibfield{author}{\bibinfo{person}{Nico Van~de Weghe},
  \bibinfo{person}{Anthony~G. Cohn}, \bibinfo{person}{Guy De~Tr\'{e}}, {and}
  \bibinfo{person}{Philippe De~Maeyer}.} \bibinfo{year}{2006}\natexlab{}.
\newblock \showarticletitle{A qualitative trajectory calculus as a basis for
  representing moving objects in Geographical Information Systems.}
\newblock \bibinfo{journal}{\emph{Control and Cybernetics}}
  \bibinfo{volume}{35}, \bibinfo{number}{1} (\bibinfo{year}{2006}),
  \bibinfo{pages}{97--119}.
\newblock


\bibitem[\protect\citeauthoryear{Van~de Weghe, Kuijpers, Bogaert, and
  Maeyer}{Van~de Weghe et~al\mbox{.}}{2005}]%
        {Weghe2005}
\bibfield{author}{\bibinfo{person}{Nico Van~de Weghe}, \bibinfo{person}{Bart
  Kuijpers}, \bibinfo{person}{Peter Bogaert}, {and}
  \bibinfo{person}{Philippe~De Maeyer}.} \bibinfo{year}{2005}\natexlab{}.
\newblock \showarticletitle{A Qualitative Trajectory Calculus and the
  Composition of Its Relations.}. In \bibinfo{booktitle}{\emph{GeoS}}
  \emph{(\bibinfo{series}{Lecture Notes in Computer Science})},
  \bibfield{editor}{\bibinfo{person}{M.~Andrea Rodríguez},
  \bibinfo{person}{Isabel~F. Cruz}, \bibinfo{person}{Max~J. Egenhofer}, {and}
  \bibinfo{person}{Sergei Levashkin}} (Eds.), Vol.~\bibinfo{volume}{3799}.
  \bibinfo{publisher}{Springer}, \bibinfo{pages}{60--76}.
\newblock
\showISBNx{3-540-30288-3}


\bibitem[\protect\citeauthoryear{Walega, Schultz, and Bhatt}{Walega
  et~al\mbox{.}}{2017}]%
        {Walega2017}
\bibfield{author}{\bibinfo{person}{Przemyslaw~Andrzej Walega},
  \bibinfo{person}{Carl P.~L. Schultz}, {and} \bibinfo{person}{Mehul Bhatt}.}
  \bibinfo{year}{2017}\natexlab{}.
\newblock \showarticletitle{Non-monotonic spatial reasoning with answer set
  programming modulo theories.}
\newblock \bibinfo{journal}{\emph{TPLP}} \bibinfo{volume}{17},
  \bibinfo{number}{2} (\bibinfo{year}{2017}), \bibinfo{pages}{205--225}.
\newblock


\bibitem[\protect\citeauthoryear{Wang, Jin, Zhang, Yang, and Ji}{Wang
  et~al\mbox{.}}{2017}]%
        {Wang2017}
\bibfield{author}{\bibinfo{person}{Zhaoyang Wang}, \bibinfo{person}{Beihong
  Jin}, \bibinfo{person}{Fusang Zhang}, \bibinfo{person}{Ruiyang Yang}, {and}
  \bibinfo{person}{Qiang Ji}.} \bibinfo{year}{2017}\natexlab{}.
\newblock \showarticletitle{Exploiting Trip Patterns in Passenger Trajectory
  Streams for Bus Scheduling Optimization in Real Time.}. In
  \bibinfo{booktitle}{\emph{Proceedings of the 18th IEEE International
  Conference on Mobile Data Management}}. \bibinfo{publisher}{IEEE Computer
  Society}, \bibinfo{pages}{266--271}.
\newblock
\showISBNx{978-1-5386-3932-0}


\bibitem[\protect\citeauthoryear{Yuan, Zheng, Xie, and Sun}{Yuan
  et~al\mbox{.}}{2013}]%
        {Yuan2013}
\bibfield{author}{\bibinfo{person}{Jing Yuan}, \bibinfo{person}{Yu Zheng},
  \bibinfo{person}{Xing Xie}, {and} \bibinfo{person}{Guangzhong Sun}.}
  \bibinfo{year}{2013}\natexlab{}.
\newblock \showarticletitle{T-Drive: Enhancing Driving Directions with Taxi
  Drivers' Intelligence.}
\newblock \bibinfo{journal}{\emph{IEEE Trans. Knowl. Data Eng.}}
  \bibinfo{volume}{25}, \bibinfo{number}{1} (\bibinfo{year}{2013}),
  \bibinfo{pages}{220--232}.
\newblock


\bibitem[\protect\citeauthoryear{Zimmermann and Freksa}{Zimmermann and
  Freksa}{1996}]%
        {Zimmermann1996}
\bibfield{author}{\bibinfo{person}{Kai Zimmermann} {and}
  \bibinfo{person}{Christian Freksa}.} \bibinfo{year}{1996}\natexlab{}.
\newblock \showarticletitle{Qualitative Spatial Reasoning Using Orientation,
  Distance, and Path Knowledge.}
\newblock \bibinfo{journal}{\emph{Appl. Intell.}} \bibinfo{volume}{6},
  \bibinfo{number}{1} (\bibinfo{year}{1996}), \bibinfo{pages}{49--58}.
\newblock


\end{thebibliography}
\end{document}